\documentclass{article}

\usepackage{microtype}
\usepackage{graphicx}
\usepackage{subfigure}
\usepackage{booktabs}

\usepackage{amsmath}
\usepackage{amsfonts}
\usepackage{amsthm}
\usepackage{thmtools}
\usepackage{thm-restate}

\usepackage{algorithm}
\usepackage{algorithmic}

\usepackage{enumitem}
\usepackage{mathrsfs}
\usepackage{bbm}
\usepackage[table]{xcolor}
\usepackage{tabularx}
\usepackage{multirow}

\usepackage{url}

\usepackage[pdfencoding=unicode,pdfusetitle,pdflang={en-US}]{hyperref}
\usepackage{pdfcomment}
\usepackage[capitalize,noabbrev]{cleveref}

\usepackage[accepted]{icml2020} %

\newtheorem{theorem}{Theorem} \crefname{theorem}{Theorem}{Theorems}
\newtheorem{lemma}[theorem]{Lemma} \crefname{lemma}{Lemma}{Lemmas}
\newtheorem{corollary}[theorem]{Corollary} \crefname{corollary}{Corollary}{Corollaries}
\newtheorem{definition}[theorem]{Definition} \crefname{definition}{Definition}{Definitions}
\newtheorem{prop}[theorem]{Proposition}  \crefname{prop}{Proposition}{Propositions}
\newtheorem{remark}[theorem]{Remark}  \crefname{remark}{Remark}{Remarks}

\newlist{assumplist}{enumerate}{1}
\setlist[assumplist]{label=(\textbf{\Alph*})}
\Crefname{assumplisti}{Assumption}{Assumptions}

\newlist{netassumplist}{enumerate}{1}
\setlist[netassumplist]{label=(\textbf{\Roman*})}
\Crefname{netassumplisti}{Assumption}{Assumptions}

\newlist{compactitem}{itemize}{3}
\setlist[compactitem]{topsep=0pt,partopsep=0pt,itemsep=4pt,parsep=0pt}
\setlist[compactitem,1]{label=\textbullet}
\setlist[compactitem,2]{label=---}
\setlist[compactitem,3]{label=*}

\newlist{compactdesc}{description}{3}
\setlist[compactdesc]{topsep=0pt,partopsep=0pt,itemsep=4pt,parsep=0pt}

\newcommand{\R}{\mathbb{R}} %
\renewcommand{\H}{\mathcal{H}} %
\newcommand{\N}{\mathcal{N}} %
\newcommand{\X}{\mathcal{X}} %
\renewcommand{\P}{\mathbb{P}} %
\newcommand{\Q}{\mathbb{Q}} %
\DeclareMathOperator{\E}{\mathbb{E}} %
\DeclareMathOperator{\Var}{Var}
\DeclareMathOperator{\bigO}{\mathcal{O}}

\newcommand{\tp}{^\mathsf{T}}

\newcommand{\nullhyp}{\mathfrak{H}_0}
\newcommand{\althyp}{\mathfrak{H}_1}

\DeclareMathOperator{\MMD}{MMD}
\DeclareMathOperator{\acc}{acc}

\newcommand{\mnstd}[2]{#1{\scriptsize$\pm$#2}}

\makeatletter
\DeclareRobustCommand{\abs}{\@ifstar\@abs\@@abs}
\newcommand{\@abs}[1]{\lvert #1 \rvert}
\newcommand{\@@abs}[1]{\lvert #1 \rvert}

\DeclareRobustCommand{\norm}{\@ifstar\@norm\@@norm}
\newcommand{\@norm}[1]{\lVert #1 \rVert}
\newcommand{\@@norm}[1]{\lVert #1 \rVert}
\makeatother

\DeclareMathOperator*{\argmax}{arg\,max}
\DeclareMathOperator*{\argmin}{arg\,min}

\newcommand{\httpsurl}[1]{\href{https://#1}{\nolinkurl{#1}}}

\begin{document}

\twocolumn[
\icmltitle{Learning Deep Kernels for Non-Parametric Two-Sample Tests}

\icmlsetsymbol{equal}{*}

\begin{icmlauthorlist}
\icmlauthor{Feng Liu}{equal,aaii,gatsby}
\icmlauthor{Wenkai Xu}{equal,gatsby}
\icmlauthor{Jie Lu}{aaii}
\icmlauthor{Guangquan Zhang}{aaii}
\icmlauthor{Arthur Gretton}{gatsby}
\icmlauthor{Danica J.\ Sutherland}{ttic}
\end{icmlauthorlist}

\icmlaffiliation{gatsby}{Gatsby Computational Neuroscience Unit, University College London, London, UK}
\icmlaffiliation{aaii}{Australian Artificial Intelligence Institute, University of Technology Sydney, Sydney, NSW, Australia}
\icmlaffiliation{ttic}{Toyota Technological Institute at Chicago, Chicago, IL, USA}

\icmlcorrespondingauthor{Feng Liu}{Feng.Liu@uts.edu.au}
\icmlcorrespondingauthor{Wenkai Xu}{wenkaix@gatsby.ucl.ac.uk}
\icmlcorrespondingauthor{Danica J.\ Sutherland}{djs@djsutherland.ml}

\icmlkeywords{Machine Learning, ICML}

\vskip 0.3in
]
\printAffiliationsAndNotice{\icmlEqualContribution} %

\begin{abstract}
We propose a class of kernel-based two-sample tests, which aim to determine whether two sets of samples are drawn from the same distribution. Our tests are constructed from kernels parameterized by deep neural nets, trained to maximize test power. These tests adapt to variations in distribution smoothness and shape over space, and are especially suited to high dimensions and complex data. By contrast, the simpler kernels used in prior kernel testing work are spatially homogeneous, and adaptive only in lengthscale. We explain how this scheme includes popular classifier-based two-sample tests as a special case, but improves on them in general. We provide the first proof of consistency for the proposed adaptation method, which applies both to kernels on deep features and to simpler radial basis kernels or multiple kernel learning. In experiments, we establish the superior performance of our deep kernels in hypothesis testing on benchmark and real-world data. The code of our deep-kernel-based two sample tests is available at \httpsurl{github.com/fengliu90/DK-for-TST}.
\end{abstract}

\section{Introduction}

Two sample tests are hypothesis tests aiming to determine whether two sets of samples are drawn from the same distribution.
Traditional methods such as $t$-tests and Kolmogorov-Smirnov tests are mainstays of statistical applications,
but require strong parametric assumptions about the distributions being studied
and/or are only effective on data in extremely low-dimensional spaces.
A broad set of
recent work in statistics and machine learning
has focused on relaxing these assumptions,
with methods either generally applicable
or specific to various more complex domains
\citep{Gretton2012,Szekely2013,Heller2016,Jitkrittum2016,RamdasGarciaCuturi,Lopez:C2ST,Chen2017,Gao18neurips,Ghoshdastidar2017,Graph_two_sample,LiW18TIT,Matthias:deep-test}.
These tests have also allowed application in various machine learning problems such as domain adaptation, generative modeling, and causal discovery \citep{MMD_GAN,Gong2016,DA_app_Stojanov,Lopez:C2ST}.

A popular class of
non-parametric two-sample tests is based on kernel methods \citep{smola1998learning}: such tests construct a
\emph{kernel mean embedding} \citep{BerTho04,Muandet2017} for each distribution, and measure the difference in these embeddings.
For any \emph{characteristic} kernel, two distributions are the same if and only if their mean embeddings are the same;
the distance between mean embeddings is the \emph{maximum mean discrepancy} (MMD) \citep{Gretton2012}.
There are also several closely related methods,
including tests based on checking for differences in mean embeddings evaluated at specific locations \citep{Chwialkowski2015,Jitkrittum2016}
and kernel Fisher discriminant analysis \citep{Harchaoui2007}.
These tests all work well for samples from simple distributions when using appropriate kernels.

Problems that we care about, however,
often involve distributions with complex structure,
where simple kernels will often map distinct distributions to nearby (and hence hard to distinguish) mean embeddings.
\Cref{fig:moti}a shows an example of a multimodal dataset,
where the overall modes align but the sub-mode structure varies differently at each mode.
A translation-invariant Gaussian kernel only ``looks at'' the data uniformly within each mode (see \cref{fig:moti}b),
requiring many samples to correctly distinguish the two distributions.
The distributions can be distinguished more effectively if we understand the structure of each mode,
as with the more complex kernel illustrated in \cref{fig:moti}c.

\begin{figure*}[tp]
    \begin{center}
        \subfigure[Samples drawn from $\P$ (left) and $\Q$ (right).]
        {\includegraphics[width=0.48\textwidth]{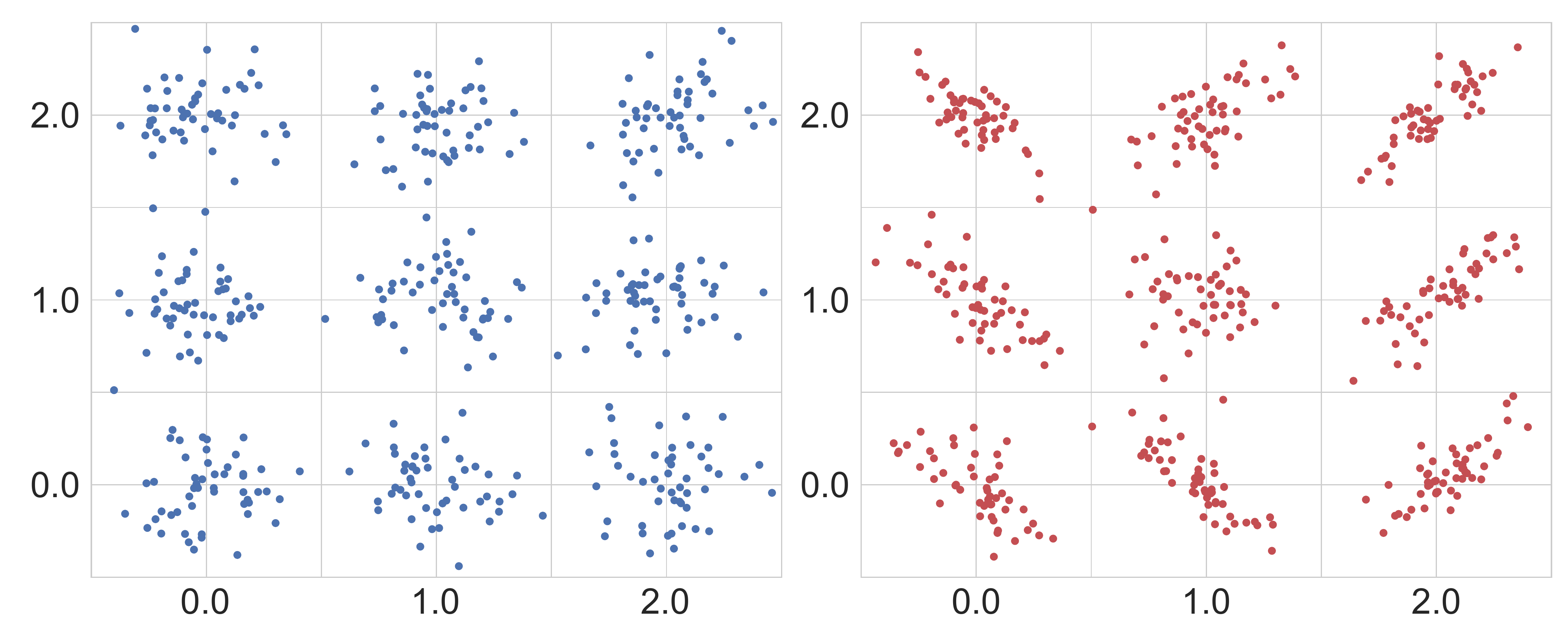}}
        \vspace{-0.3cm}
        \subfigure[Contour of Gaussian]
        {\includegraphics[width=0.24\textwidth]{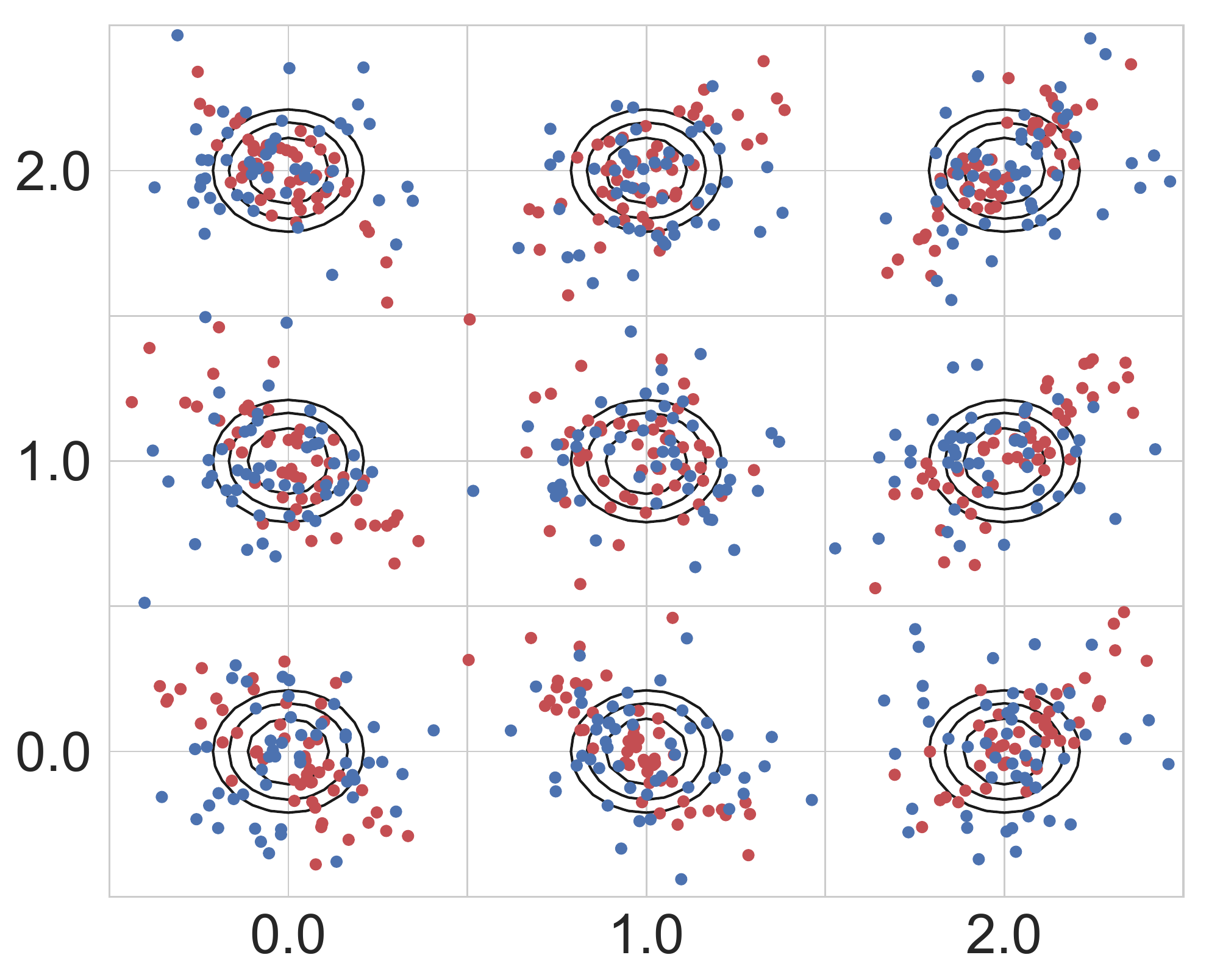}}
        \subfigure[Contour of deep kernel]
        {\includegraphics[width=0.24\textwidth]{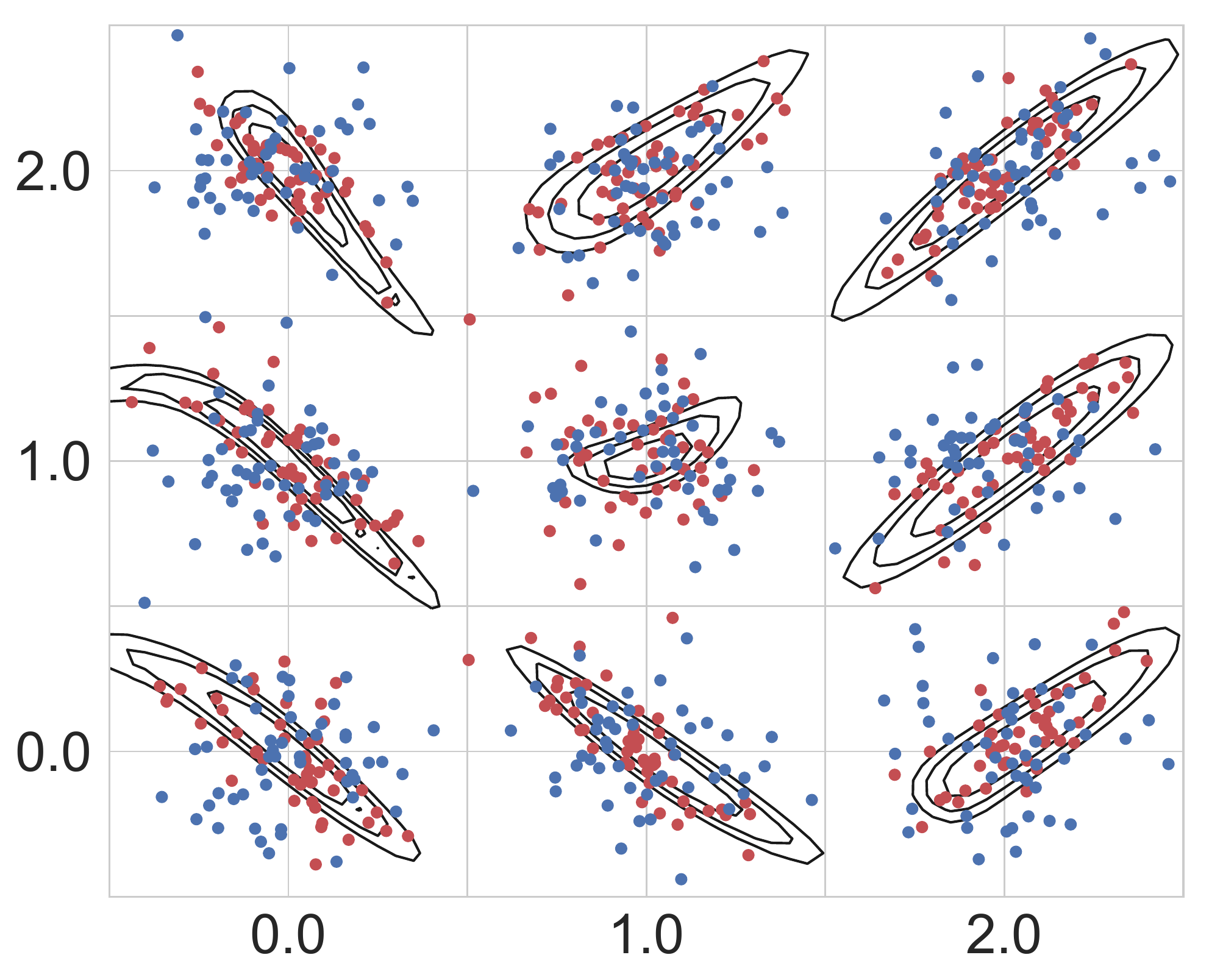}}
        \caption{%
        In the Blob dataset,
        $\P$ and $\Q$ are each equal mixtures of nine Gaussians with the same modes (a),
        but each component of $\P$ is an isotropic Gaussian
        whereas the covariance of $\Q$ differs in each component.
        Panels (b) and (c) show the contours of a kernel,
        $k(x, \mu_i)$ for each of the nine modes $\mu_i$;
        contour values are $0.7$, $0.8$ and $0.9$.
        A Gaussian kernel (b) treats points isotropically throughout the space, based only on $\lVert x - y \rVert$.
        A deep kernel (c) learned by our methods behaves differently in different parts of the space, adapting to the local structure of the data distributions and hence allowing better identification of differences between $\P$ and $\Q$.}
        \label{fig:moti}
    \end{center}
    \vspace{-1em}
\end{figure*}

To model these complex functions,
we adopt a \emph{deep kernel} approach \citep{wilson:deep-kernel-learning,sutherland:mmd-opt,Li2017,Jean2018,Kevin_ICML2019},
building a kernel with a deep network.
In this paper, we use
\begin{align}
\label{eq:deepkernel_simpleForm}
    k_\omega(x,y) = [(1-\epsilon)\kappa(\phi_\omega(x),\phi_\omega(y))+\epsilon]q(x,y),
\end{align}
where the deep neural network $\phi_\omega$ extracts features of samples,
and $\kappa$ is a simple kernel (e.g., a Gaussian) on those features,
while $q$ is a simple characteristic kernel (e.g. Gaussian) on the input space.
With an appropriate choice of $\phi_\omega$,
this allows for extremely flexible kernels
which can learn complex behavior very different in different parts of space.
This choice is discussed further in \cref{sec:DKforTST}.

These complex kernels, though, cannot feasibly be specified by hand or simple heuristics, as is typical practice in kernel methods.
We select the parameters $\omega$
by maximizing the ratio of the MMD to its variance,
which maximizes test power at large sample sizes.
This procedure was proposed by \citet{sutherland:mmd-opt},
but we establish for the first time that it gives consistent selection of the best kernel in the class,
whether optimizing our deep kernels with hundreds of thousands of parameters
or simply choosing lengthscales of a Gaussian as did \citeauthor{sutherland:mmd-opt}
Previously, there were no guarantees this procedure would yield a kernel
which generalized at all from the training set to a test set.

Another way to compare distributions is to train a classifier between them, and evaluate its accuracy \citep{Lopez:C2ST}.
We show, perhaps surprisingly, that our framework encompasses this approach,
but deep kernels allow for more general model classes which can use the data more efficiently.
We also train representations
directly to maximize test power,
rather than a cross-entropy surrogate.

We test our method on several simulated and real-world datasets,
including complex synthetic distributions,
high-energy physics data,
and challenging image problems.
We find convincingly that learned deep kernels outperform simple shallow methods,
and learning by maximizing test power
outperforms learning through a cross-entropy surrogate loss.

\section{MMD Two-Sample Tests}\label{sec:background}
\paragraph{Two-sample testing.}
Let $\X$ be a separable metric space
-- in this paper, typically a subset of $\R^d$ --
and $\P$, $\Q$ be Borel probability measures on $\X$.
We observe independent identically distributed (\emph{i.i.d.}) samples $S_\P=\{x_i\}_{i=1}^n \sim \P^n$ and $S_\Q=\{y_j\}_{j=1}^m \sim \Q^m$.
We wish to know whether $S_\P$ and $S_\Q$ come from the same distribution:
does $\P=\Q$?

We use the null hypothesis testing framework,
where  the null hypothesis $\nullhyp: \P=\Q$ is tested against the alternative hypothesis $\althyp: \P \neq \Q$.
We perform a two-sample test in four steps:
select a significance level $\alpha \in [0,1]$;
compute a test statistic $\hat{t}(S_\P, S_\Q)$;
compute the $p$-value $\hat{p}= \Pr_{\nullhyp}(T>\hat{t})$, the probability of the two-sample test returning a statistic as large as $\hat{t}$ when $\nullhyp$ is true;
finally, reject $\nullhyp$ if $\hat{p} < \alpha$.

\paragraph{Maximum mean discrepancy (MMD).}
We will base our two-sample test statistic
on an estimate of a distance between distributions.
Our metric, the MMD,
is defined in terms of a kernel $k$
giving point-level ``similarities'' on $\X$.
\begin{definition}[\citealp{Gretton2012}]
Let $k:\X \times \X \to \R$
be the kernel
of a reproducing kernel Hilbert space $\mathcal{H}_{k}$,
with feature maps $k(\cdot, x) \in \mathcal{H}_{k}$.
Let $X, X' \sim \P$ and $Y, Y' \sim \Q$,
and define the \emph{kernel mean embeddings}
$\mu_\P := \E[ k(\cdot, X) ]$ and $\mu_\Q := \E[ k(\cdot, Y) ]$.
Under mild integrability conditions,
\begin{multline*}
\MMD(\P,\Q;\H_k)
  := \sup_{f\in\H, \norm{f}_{\H_k} \le 1} \abs{ \E[ f(X) ] - \E[ f(Y) ] }
\\ = \norm{\mu_\P - \mu_\Q}_{\H_k}
   = \sqrt{ \E\left[ k(X, X') + k(Y, Y') - 2 k(X, Y) \right] }
.\end{multline*}
For \emph{characteristic} kernels,
$\mu_\P = \mu_\Q$ implies $\P = \Q$,
hence $\MMD(\P, \Q; \H_k) = 0$ if and only if $\P = \Q$.
\end{definition}

The first form shows that the MMD is an integral probability metric \citep{muller1997integral},
along with such popular distances as the Wasserstein and total variation.

There are several natural estimators of the MMD from samples.
We will assume $n = m$
and use the $U$-statistic estimator,
which is unbiased for $\MMD^2$ and has nearly minimal variance among unbiased estimators \citep{Gretton2012}:
\begin{gather}
\widehat{\MMD}_u^2(S_\P, S_\Q; k)
:= \frac{1}{n (n-1)} \sum_{i \ne j} H_{ij}
\label{eq:MMD_U_compute}
\\
H_{ij} :=
    k(X_i, X_j)
    + k(Y_i, Y_j)
    - k(X_i, Y_j)
    - k(Y_i, X_j)
\notag
.\end{gather}
The similar $\widehat\MMD_b^2 := \frac{1}{n^2} \sum_{i j} H_{ij}$ is
the squared MMD between the empirical distributions of $S_\P$ and $S_\Q$.\footnote{Including $k(X_i, Y_i)$ terms in $\widehat\MMD_u$ gives the minimal variance unbiased estimator, and allows $m \ne n$. The $U$-statistic is more convenient for analysis and for efficient permutations; in our settings it behaves similarly to the MVUE and $\widehat\MMD_b^2$.}{}

\paragraph{Testing with the MMD.}
It can be shown that under $\nullhyp$,
$n \widehat{\MMD}_u^2$ converges to a distribution depending on $\P$ and $k$;
we thus use this as our test statistic.

\begin{prop}[Asymptotics of $\widehat{\MMD}_u^2$] \label{prop:asymptotics}
Under the null hypothesis, $\nullhyp: \P=\Q$,
we have
if $Z_i \sim \mathcal N(0, 2)$,
\[
n\widehat\MMD_u^2 \overset{\textbf{d}}{\to} \sum_i \sigma_i(Z^2_i - 2);
\]
here $\sigma_i$ are
the eigenvalues of the $\P$-covariance operator of the centered kernel
\citep[Theorem 12]{Gretton2012},
and $\overset{\textbf{d}}{\to}$ denotes convergence in distribution.

Under the alternative, $\althyp : \P \ne \Q$,
a standard central limit theorem holds
\citep[Section 5.5.1]{serfling}:
\begin{gather*}
\sqrt{n}(\widehat{\MMD}_u^2 - \MMD^2) \overset{\textbf{d}}{\to} \N(0, \sigma^2_{\althyp})
\\ \sigma_{\althyp}^2 := 4 \left( \E[H_{12} H_{13}] - \E[H_{12}]^2 \right)
\end{gather*}
where $H_{12}$, $H_{13}$ refer to $H_{ij}$ above.
\end{prop}

Although it is possible to construct a test based on directly estimating this null distribution \citep{eig-mmd-null},
it is both simpler and, if implemented carefully, faster \citep{sutherland:mmd-opt} to instead use a permutation test.
This general method \citep{dwass1957,AlbaFernandez2008} observes that under $\nullhyp$,
the samples from $\P$ and $\Q$ are interchangeable;
we can therefore estimate the null distribution of our test statistic
by repeatedly re-computing it with the samples randomly re-assigned to $S_\P$ or $S_\Q$.

\paragraph{Test power.}
The main measure of efficacy of a null hypothesis test is its \emph{power}:
the probability that, for a particular $\P \ne \Q$ and $n$, we correctly reject $\nullhyp$.
\Cref{prop:asymptotics} implies,
where $\Phi$ is the standard normal CDF,
that
\[
    {\Pr}_{\althyp}\left( n \widehat{\MMD}_u^2 > r \right)
    \to \Phi\left( \frac{\sqrt{n} \MMD^2}{\sigma_{\althyp}} - \frac{r}{\sqrt{n} \, \sigma_{\althyp}} \right)
;\]
we can find the approximate test power by using the rejection threshold, found via (e.g.) permutation testing, as $r$.
We also know via \cref{prop:asymptotics} that this $r$ will converge to a constant,
and $\MMD$, $\sigma_{\althyp}$ are also constants.
For reasonably large $n$,
the power is dominated by the first term,
and the kernel yielding the most powerful test will approximately maximize \citep{sutherland:mmd-opt}
\begin{equation} \label{eq:tpp}
    J(\P, \Q; k) := \MMD^2(\P, \Q; k) / \sigma_{\althyp}(\P, \Q; k)
.\end{equation}

\paragraph{Selecting a kernel.}
The criterion $J(\P, \Q; k)$ depends on the particular $\P$ and $\Q$ at hand,
and thus we typically will neither be able to choose a kernel \emph{a priori},
nor exactly evaluate $J$ given samples.
We can, however, estimate it with
\begin{equation} \label{eq:tpp-hat}
  \hat J_\lambda(S_\P, S_\Q; k) := \frac{\widehat\MMD_u^2(S_\P, S_\Q; k)}{\hat\sigma_{\althyp,\lambda}(S_\P, S_\Q; k)}
,\end{equation}
where $\hat\sigma_{\althyp,\lambda}^2$ is a regularized estimator of $\sigma_{\althyp}^2$
given by\footnote{%
This estimator, as a $V$-statistic, is biased even when $\lambda = 0$
(although this bias is only $O(1/N)$; see \cref{thm:var-est-bias}).
Although \citet{sutherland:mmd-opt,unbiased-var-ests} give a quadratic-time estimator unbiased for $\sigma_{\althyp}^2$,
it is much more complicated to implement and analyze,
likely has higher variance,
and (being unbiased) can be negative,
especially e.g.\ when the kernel is poor.}
\begin{equation} \label{eq:estimate_sigma_H1}
    \frac{4}{n^3} \sum_{i=1}^n \left( \sum_{j=1}^n H_{ij} \right)^2
    - \frac{4}{n^4}\left( \sum_{i=1}^n \sum_{j=1}^n H_{ij} \right)^2
    + \lambda
.\end{equation}

Given $S_\P$ and $S_\Q$,
we could construct a test by choosing $k$ to maximize $\hat J_\lambda(S_\P, S_\Q; k)$,
then using a test statistic based on $\widehat\MMD(S_\P, S_\Q; k)$.
This sample re-use, however,
violates the conditions of \cref{prop:asymptotics},
and permutation testing would require repeatedly re-training $k$ with permuted labels.

Thus we split the data,
get $k^{tr} \approx \argmax_k \hat J_\lambda(S_\P^{tr}, S_\Q^{tr}; k)$,
then compute the test statistic and permutation threshold on $S_\P^{te}$, $S_\Q^{te}$ using $k^{tr}$.
This procedure was proposed for $\widehat\MMD_u^2$ by \citet{sutherland:mmd-opt},
but the same technique works for a variety of tests
\citep{Gretton2012NeurIPS,Jitkrittum2016,Jitkrittum2017,Lopez:C2ST}.
Our paper adopts this framework (\cref{sec:DKforTST})
and studies it further.

\paragraph{Relationship to other approaches.}
One common scheme is to pick a kernel $k_\omega$ based on some proxy task,
such as a related classification problem
(e.g.\ \citealt{Matthias:deep-test} or the KID score of \citealt{MMD_GAN}).
Although this approach can work quite well,
it depends entirely on features from the proxy task applying well to the differences between $\P$ and $\Q$,
which can be hard to know in general.

An alternative is to maximize simply $\widehat\MMD_u$
(\citealt{sriperumbudur2009choice}; proposed but not evaluated by \citeauthor{Matthias:deep-test}).
Ignoring $\sigma_{\althyp}$ means that,
for instance, this approach would choose to simply scale $k \to \infty$,
even though this does not change the test at all.
Even when this is not possible,
\citet{sutherland:mmd-opt} found this approach notably worse than maximizing \eqref{eq:tpp-hat};
we confirm this in our experiments.

MMD-GANs \citep{Li2017,MMD_GAN}
also simply maximize $\widehat\MMD_u$
to identify the differences between their model $\Q_\theta$ and target $\P$.
If $\Q_\theta$ is quite far from $\P$, however,
an MMD-GAN requires a ``weak'' kernel to identify a path for improving $\Q_\theta$ \citep{MMD_GAN2},
while our ideal kernel is one which perfectly distinguishes $\P$ and $\Q_\theta$ and would likely give no signal for improvement.
Our algorithm, theoretical guarantees, and empirical evaluations thus all differ significantly from those for MMD-GANs.

\section{Limits of Simple Kernels}
\label{sec:MMD_limited}

We can use the criterion $\hat J_\lambda$ of \eqref{eq:tpp-hat}
even to select parameters among a simple family,
such as the lengthscale of a Gaussian kernel.
Doing so on the \emph{Blob} problem of \cref{fig:moti}
illustrates the limitations of using MMD with these kernels.

In \cref{fig:Blob_RES}c,
we show how the maximal value of $\hat J$
changes as we see more samples from $\P$ and $\Q$,
for both a family of Gaussian kernels (green dashed line)
and a family \eqref{eq:deepkernel_simpleForm} of deep kernels (red line).
The optimal $\hat J$ is always higher for the deep kernels;
as expected, the empirical test power (\cref{fig:Blob_RES}a) is also higher for deep kernels.

Most simple kernels used for MMD tests,
whether the Gaussian we use here or Laplace, inverse multiquadric,
even automatic relevance determination kernels,
are all translation invariant:
$k(x, y) = k(x - t, y - t)$ for any $t \in \R^d$.
(All kernels used by \citet{sutherland:mmd-opt}, for instance, were of this type.)
Hence the kernel behaves the same way across space,
as in \cref{fig:moti}b.
This means that for distributions whose behavior varies through space,
whether because principal directions change (as in \cref{fig:moti}) so the shape should be different,
or because some regions are much denser than others and so need a smaller lengthscale \citep[e.g.][Figures 1 and 2]{Kevin_ICML2019},
any single global choice is suboptimal.%

Kernels which are not translation invariant,
such as the deep kernels \eqref{eq:deepkernel_simpleForm}
shown in \cref{fig:moti}c,
can adapt to the different shapes necessary in different areas.

\begin{figure*}[!t]
    \begin{center}
    \small
        \subfigure
        {\pdftooltip{\includegraphics[width=0.77\textwidth]{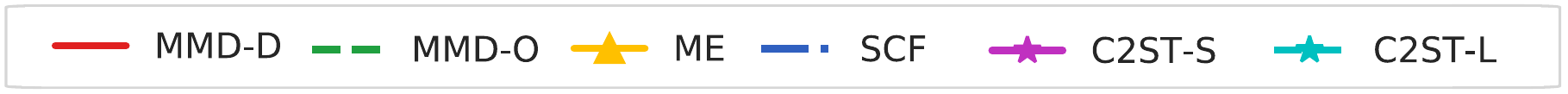}}{These are baselines considered in this paper. See more details in Section~\ref{sec:exp}.}}
        \subfigure[Average test power]
        {\includegraphics[width=0.25\textwidth]{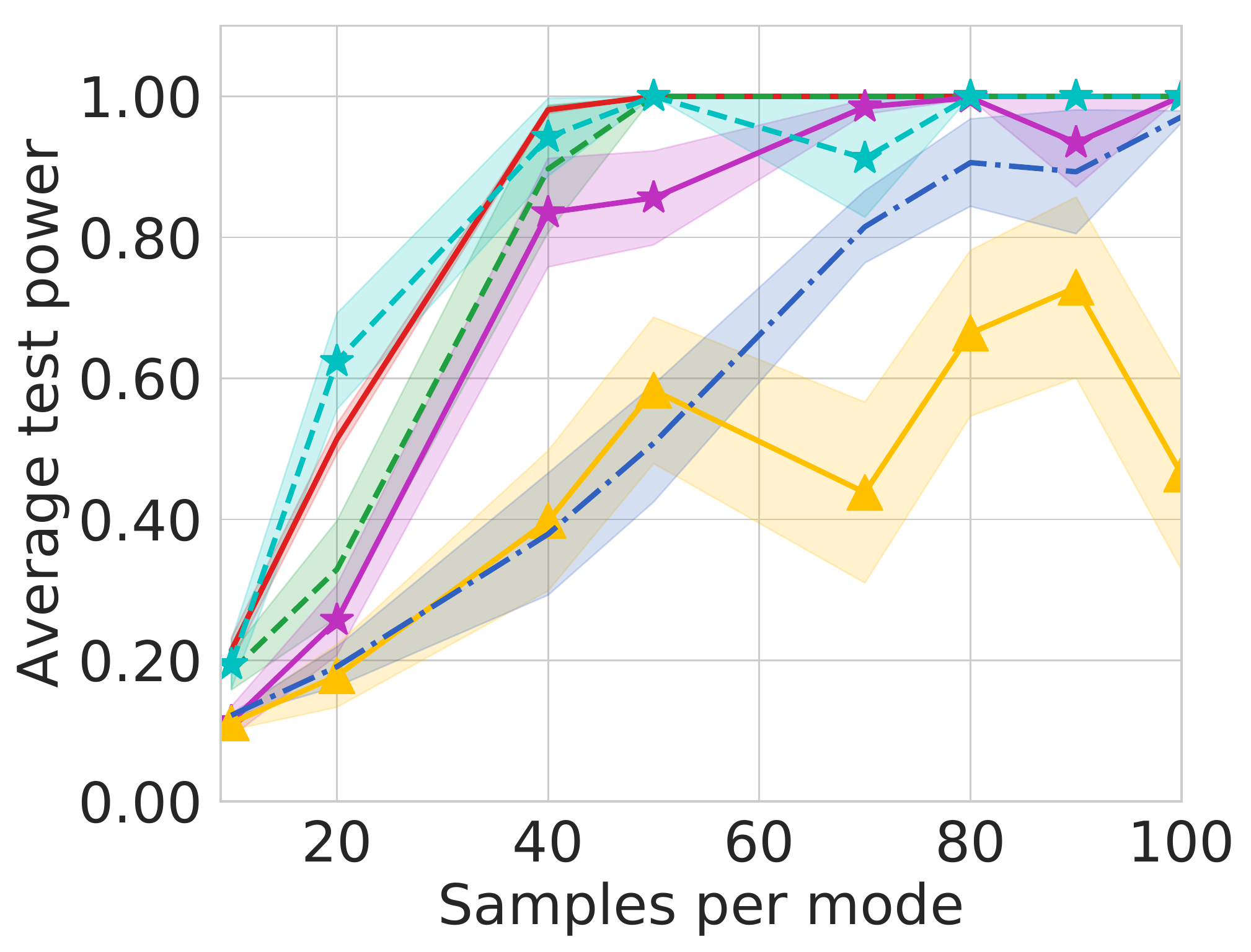}}
        \vspace{-0.3cm}
        \subfigure[STD of test power]
        {\includegraphics[width=0.25\textwidth]{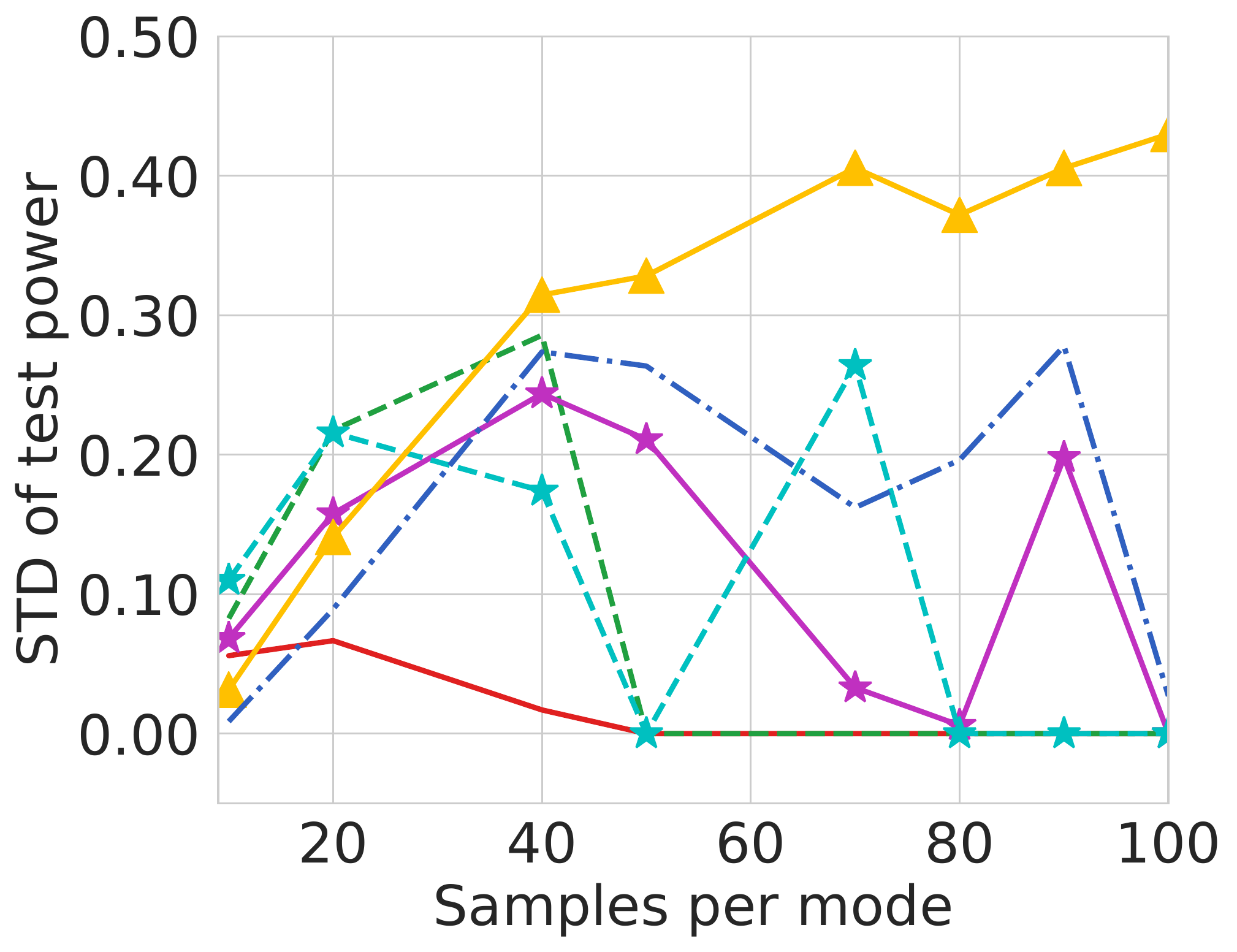}}
        \subfigure[$\hat J_\lambda$ value \eqref{eq:tpp-hat}]
        {\includegraphics[width=0.24\textwidth]{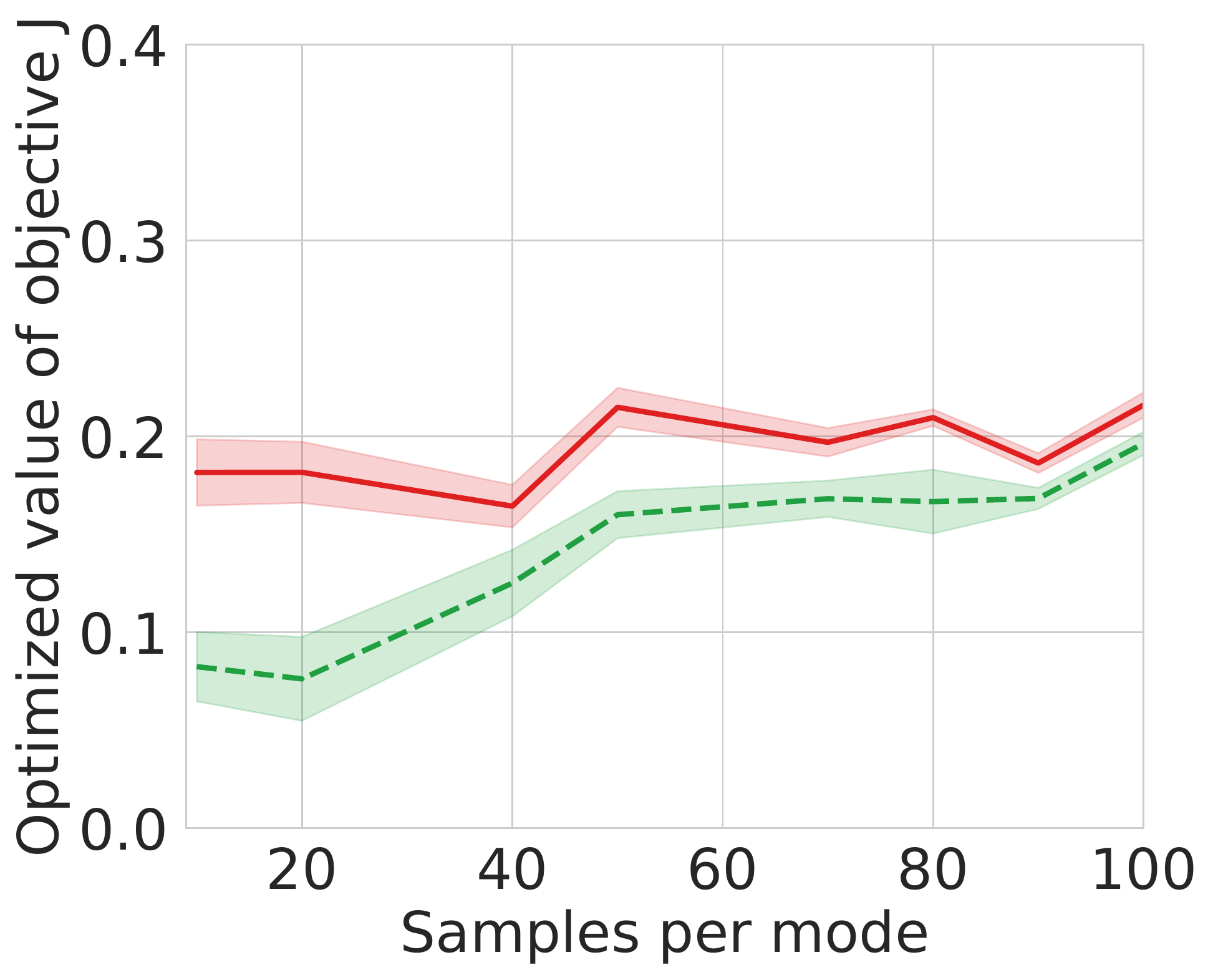}}
        \subfigure[Type-I error]
        {\includegraphics[width=0.24\textwidth]{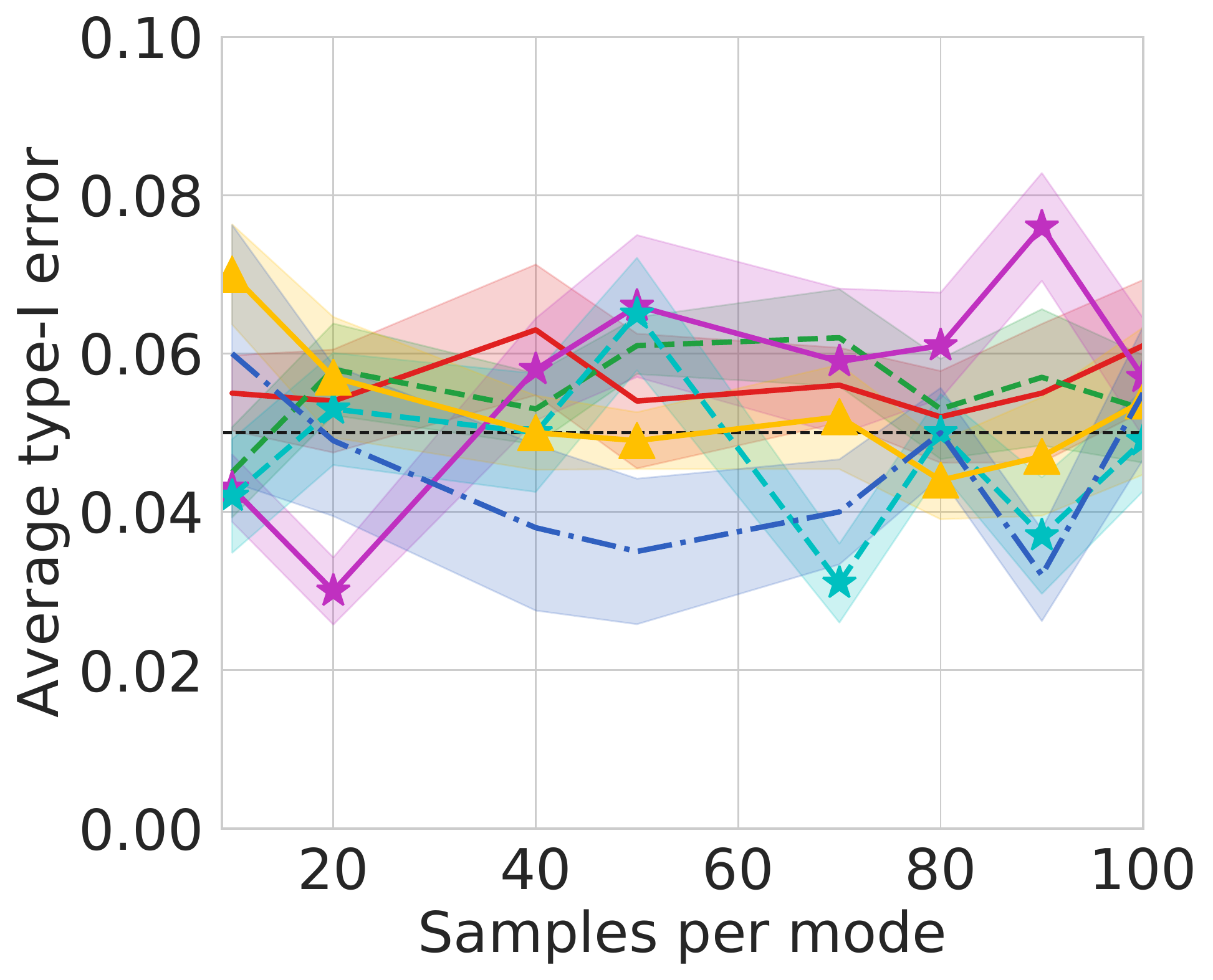}}
        \caption{Results on \emph{Blob-S} and \emph{Blob-D} given $\alpha=0.05$; see \cref{sec:exp} for details. $n_b$ is the number of samples at each mode, so $n_b = 100$ means drawing $900$ samples from each of $\P$ and $\Q$. We report, when increasing $n_b$, (a) average test power, (b) standard deviation of test power, (c) the value of $\hat J_\lambda$, and (d) average type-I error. (a), (b) and (c) are on \emph{Blob-D}, and (d) is on \emph{Blob-S}. Shaded regions show standard errors for the mean, and the black line shows $\alpha$.}
        \label{fig:Blob_RES}
    \end{center}
    \vspace{-1em}
\end{figure*}

\section{Relationship to Classifier-Based Tests} \label{sec:c2st-relation}
Another popular method for conducting two-sample tests is to train a classifier between $S_\P^{tr}$ and $S_\Q^{tr}$,
then assess its performance on $S_\P^{te}$, $S_\Q^{te}$.
If $\P = \Q$, the classification problem is impossible and performance will be at chance.

The most common performance metric is the accuracy \citep{Lopez:C2ST};
this scheme is fairly common among practitioners, and
\citet{Ramdas:clf} showed it to be optimal in rate, but suboptimal in constant, in one limited setting
(linear discriminant analysis between high-dimensional elliptical distributions, e.g.\ Gaussians, with identical covariances).
We will call this approach a Classifier Two-Sample Test based on Sign, C2ST-S.
Letting $f : \X \to \R$ output classification scores,
the C2ST-S statistic
is $\widehat{\acc}(S_\P, S_\Q; f)$
given by
\[
    \frac{1}{2 n} \sum_{X_i \in S_\P} \mathbbm{1}(f(X_i) > 0)
  + \frac{1}{2 n} \sum_{Y_i \in S_\Q} \mathbbm{1}(f(Y_i) \le 0)
.\]
Let $\acc(\P, \Q; f) := \frac12 \Pr(f(X) > 0) + \frac12 \Pr(f(Y) \le 0)$;
$\widehat\acc$ is unbiased for $\acc$ and has a simple asymptotically normal null distribution.

Although it is perhaps not immediately obvious this is the case,
C2ST-S is almost a special case of the MMD.
Let
    \begin{equation}
        k_f^{(S)}(x, y)
        = \frac14 \, \mathbbm{1}(f(x) > 0) \, \mathbbm{1}(f(y) > 0)
    \label{eq:sign-kernel}
    .\end{equation}
A C2ST-S with $f$ is equivalent to an MMD test with $k_f^{(S)}$:
\begin{prop} \label{thm:c2st-equiv}
    It holds that
    \begin{gather*}
        \MMD(\P, \Q; k_f^{(S)})
        = \abs{ \acc(\P, \Q; f) - \frac12 }
    \\
        \widehat\MMD_b(S_\P, S_\Q; k_f^{(S)})
        = \abs{ \widehat\acc(S_\P, S_\Q; f) - \frac12 }
    .\end{gather*}
\end{prop}
\begin{proof}
    The mean embedding $\mu_\P$ under $k_f^{(S)}$
    is simply $\frac{1}{2} \E \mathbbm{1}(f(X) > 0) = \frac{1}{2} \Pr(f(X) > 0)$,
    so the MMD is
    \[
        \frac12 \Big\lvert \Pr(f(X) > 0) - \Pr(f(Y) > 0) \Big\rvert
        = \Big\lvert \acc(\P, \Q; f) - \frac12 \Big\rvert
    .\]
    Moreover, $\widehat\acc$ is $\acc$ on empirical distributions.
\end{proof}
The C2ST-S, however, selects $f$ to maximize cross-entropy (approximately maximizing $\widehat\acc$),
while we maximize $\hat J_\lambda$ \eqref{eq:tpp-hat}.
Although $k_f^{(S)}$ is not differentiable,
maximizing \eqref{eq:tpp} would exactly maximize $\acc$
and hence maximize test power \citep[Theorem 1]{Lopez:C2ST}.

Accessing $f$ only through its sign allows for a simple null distribution,
but it ignores $f$'s measure of confidence:
a highly confident output extremely far from the decision boundary
is treated the same as a very uncertain one lying in an area of high overlap between $\P$ and $\Q$,
dramatically increasing the variance of the statistic.
A scheme we call C2ST-L instead
tests difference in means of $f$ on $\P$ and $\Q$ \citep{cheng:net-logits}.
Let
    \begin{equation}
        k_f^{(L)}(x, y)
        = f(x) f(y)
    \label{eq:lin-kernel}
    .\end{equation}
A C2ST-L is equivalent to an MMD test with $k_f^{(L)}$:
\begin{prop} \label{thm:c2st-l-equiv}
    It holds that
    \begin{gather*}
        \MMD(\P, \Q; k_f^{(L)})
        = \abs{ 
            \E f(X) - \E f(Y)
        }
    \\
        \widehat\MMD_b(S_\P, S_\Q; k_f^{(L)})
        = \abs{
            \frac{1}{n} \sum_{X_i \in S_\P} f(X_i)
          - \frac{1}{n} \sum_{Y_i \in S_\Q} f(Y_i)
        }
    .\end{gather*}
\end{prop}
\begin{proof}
    This kernel's feature map is $k_f^{(L)}(x, \cdot) = f(x)$.
\end{proof}
Now maximizing accuracy (or a cross-entropy proxy) no longer directly maximizes power.
This kernel is differentiable,
so we can directly compare the merits of maximizing \eqref{eq:tpp-hat} to maximizing cross-entropy;
we will see in \cref{sec:tpp-vs-ce} that our more direct approach
is empirically superior.

Compared to using $k_f^{(L)}$, however,
\cref{sec:tpp-vs-ce} shows that
learned MMD tests also obtain better performance using kernels like \eqref{eq:deepkernel_simpleForm}.
This is analogous to a similar phenomenon observed in other problems by \citet{MMD_GAN} and \citet{Kevin_ICML2019}:
C2STs learn a full discriminator function on the training set,
and then apply only that function to the test set.
Learning a deep kernel like \eqref{eq:deepkernel_simpleForm} corresponds to learning only a powerful \emph{representation} on the training set,
and then \emph{still learning} $f$ itself from the test set
-- in a closed form that makes permutation testing simple.

\section{Learning Deep Kernels}\label{sec:DKforTST}

\paragraph{Choice of kernel architecture.} \label{sec:kernel-arch}
Most previous work on deep kernels %
has used a kernel $\kappa$ directly on the output of a featurization network $\phi_\omega$,
$k_\omega(x, y) = \kappa(\phi_\omega(x), \phi_\omega(y))$.
This is certainly also an option for us.
Any such $k_\omega$, however, is characteristic if and only if $\phi_\omega$ is injective.
If we select our kernel well, this is not really a concern.\footnote{%
A characteristic kernel on top of even
$\phi_\omega(x) = \omega\tp x$
with a \emph{random} $\omega$
will be almost surely consistent \citep{Heller2016},
and in general the existence of even one good $\phi_\omega$ for a particular $\P$, $\Q$ pair is enough that a perfect optimizer would be able to distinguish the distributions
\citep[Proposition 1]{MMD_GAN2}.}{}
Even so, it would be reassuring to know that,
even if the optimization goes awry,
the resulting test will still be at least consistent.
More importantly,
it can be helpful in optimization to add a ``safeguard'' preventing the learned kernel from considering extremely far-away inputs as too similar.
We can achieve these goals with the form \eqref{eq:deepkernel_simpleForm},
repeated here:
\begin{align*}
    k_\omega(x,y) = [(1-\epsilon)\kappa(\phi_\omega(x),\phi_\omega(y))+\epsilon] \, q(x,y)
.\end{align*}
Here $\phi_\omega$ is a deep network (with parameters $\omega$) that extracts features,
and $\kappa$ is a kernel on those features;
we use a Gaussian with lengthscale $\sigma_\phi$,
$\kappa(a, b) = \exp\left(-\frac{1}{2 \sigma_\phi^2} \norm{a - b}^2 \right)$.
We choose $0 < \epsilon < 1$ and
$q$ a Gaussian with lengthscale $\sigma_q$.
\begin{prop} %
Let $k_\omega$ be of the form \eqref{eq:deepkernel_simpleForm}
with $\epsilon > 0$ and $q$ characteristic.
Then $k_\omega$ is characteristic.
\end{prop}

\paragraph{Learning the deep kernel.}
The kernel optimization and testing procedure is summarized in \cref{alg:learn_deep_kernel}.
For larger datasets, or when $n \ne m$,
we use minibatches in the training procedure;
for smaller datasets, we use full batches.
We use the Adam optimizer \citep{Adam:optimizer}.
Note that the parameters $\epsilon$, $\sigma_\phi$, and $\sigma_q$
are included in $\omega$,
all parameterized in log-space
(i.e. we optimize $\epsilon'$ where $\epsilon = \exp(\epsilon')$).

\begin{algorithm}[tb]
\footnotesize
\caption{Testing with a learned deep kernel}
\label{alg:learn_deep_kernel}
\begin{algorithmic}
\STATE \textbf{Input:} $S_\P$, $S_\Q$, various hyperparameters used below;

\vspace{1mm}
\STATE $\omega \gets \omega_0$; $\lambda \gets 10^{-8}$; %

\STATE Split the data as $S_\P = S^{tr}_\P\cup S^{te}_\P$ and $S_\Q = S^{tr}_\Q\cup S^{te}_\Q$;

\vspace{1mm}
\STATE \textit{\# Phase 1: train the kernel parameters $\omega$ on $S^{tr}_\P$ and $S^{tr}_\Q$\hfill}

\FOR{$T = 1,2,\dots,T_\mathit{max}$}

\STATE $X \gets$ minibatch from $S^{tr}_\P$; $Y \gets$ minibatch from $S^{te}_\Q$;

\STATE $k_\omega \gets$ kernel function with parameters $\omega$;
       \hfill \textit{\# as in \eqref{eq:deepkernel_simpleForm}}

\STATE $M(\omega) \gets \widehat\MMD_u^2(X, Y; k_\omega)$;
       \hfill\textit{\# using \eqref{eq:MMD_U_compute}}

\STATE $V_\lambda(\omega) \gets \hat{\sigma}^2_{\althyp,\lambda}(X, Y; k_\omega)$;
       \hfill\textit{\# using \eqref{eq:estimate_sigma_H1}}

\STATE $\hat J_\lambda(\omega) \gets {M(\omega)}/\sqrt{V_\lambda(\omega)}$;
       \hfill\textit{\# as in \eqref{eq:tpp-hat}}

\STATE $\omega \gets \omega + \eta\nabla_{\textnormal{Adam}} \hat J_\lambda(\omega)$;
       \hfill \textit{\# maximize $\hat J_\lambda(\omega)$}

\ENDFOR

\vspace{1mm}
\STATE \textit{\# Phase 2: permutation test with $k_\omega$ on $S^{te}_\P$ and $S^{te}_\Q$}
\STATE $\mathit{est} \gets \widehat\MMD_u^2(S^{te}_\P, S^{te}_\Q; k_\omega)$

\FOR{$i = 1, 2, \dots, n_\mathit{perm}$}
\STATE Shuffle $S^{te}_\P \cup S^{te}_\Q$ into $X$ and $Y$
\STATE $\mathit{perm}_i \gets \widehat\MMD_u^2(X, Y; k_\omega)$
\ENDFOR

\STATE \textbf{Output:} $k_\omega$, $\mathit{est}$, $p$-value $\frac{1}{n_\mathit{perm}} \sum_{i=1}^{n_\mathit{perm}} \mathbbm{1}(\mathit{perm}_i \ge \mathit{est})$
\end{algorithmic}
\end{algorithm}

\paragraph{Time complexity.}
Let $E$ denote the cost of computing an embedding $\phi_\omega(x)$,
and $K$ the cost of computing \eqref{eq:deepkernel_simpleForm} given $\phi_\omega(x)$, $\phi_\omega(y)$.
Then each iteration of training in \cref{alg:learn_deep_kernel}
costs $\mathcal{O}\left( m E + m^2 K \right)$, where $m$ is the minibatch size;
for the moderate $m$ that fit in a GPU-sized minibatch anyway,
the $m E$ term typically dominates,
matching the complexity of a C2ST.
Testing takes time $\mathcal{O}\left( n E + n^2 K + n^2 \, n_\mathit{perm} \right)$,
compared to $\mathcal{O}\left( n E + n \, n_\mathit{perm} \right)$ for permutation-based C2STs.
In either case, the quadratic factors could if necessary be reduced with the block estimator approach of \citet{Blaschko2013}, at the cost of some test power.
In our experiments in \cref{sec:exp}, the overall runtime of our methods was scarcely different from the overall runtime of C2STs.

\section{Theoretical Analysis}
We now show that optimizing the regularized test power criterion based on a finite number of samples works:
as $n$ increases, our estimates converge uniformly over a ball in parameter space,
and therefore if there is a unique best kernel,
we converge to it.
\citet{sutherland:mmd-opt} gave no such guarantees;
this result allows us to trust that,
at least for reasonably large $n$ and if our optimization process succeeds,
we will find a kernel that generalizes nearly optimally
rather than just overfitting to $S^{tr}$.

We first state a generic result,
then show some choices of kernels, particularly deep kernels \eqref{eq:deepkernel_simpleForm}, satisfy the conditions.
\begin{theorem} \label{thm:test-power-conv}
   Let $\omega$ parameterize uniformly bounded kernel functions $k_\omega$ in a Banach space of dimension $D$,
   with
   $\abs{ k_\omega(x, y) - k_{\omega'}(x, y) } \le L_k \norm{\omega - \omega'}$.
   Let $\bar\Omega_s$ be a set of $\omega$ for which
   $\sigma_{\althyp}^2(\P, \Q; k_\omega) \ge s^2 > 0$
   and $\norm \omega \le R_\Omega$.
   Take $\lambda = n^{-1/3}$.
   Then,
   with probability at least $1 - \delta$,
   \begin{multline*}
       \sup_{\omega \in \bar\Omega_s} \abs*{
            \hat J_{\lambda}(S_\P, S_\Q; k_\omega)
          - J(\P, \Q; k_\omega)
        } =
        \\ \bigO\left(\frac{1}{s^2 n^{1/3}} \left[ \frac1s + \sqrt{D \log(R_\Omega n) + \log\frac1\delta} + L_k \right]
        \right)
    .\end{multline*}
    If there is a unique best kernel $\omega^*$,
    the maximizer of $\hat J_\lambda$ converges in probability to $\omega^*$ as $n \to \infty$.
\end{theorem}
A version with explicit constants and more details is given in \cref{Asec:proof} (as \cref{thm:ratio-conv,thm:param-conv});
the proof is based on
uniform convergence of the MMD and variance estimators
using an $\epsilon$-net argument.

The following results are shown in \cref{sec:proof:kernel-props}.
We first show a result on simple Gaussian bandwidth selection.

\begin{prop} \label{thm:main-rbf-lip}
    Suppose each $x \in \X$ has $\norm x \le R_X$,
    and we choose the bandwidth of a Gaussian kernel
    among a set whose minimum is at least $1 / R_\Omega$.
    Then the conditions of \cref{thm:test-power-conv}
    are met with $D = 1$ and
    $L_k = 2 R_X / \sqrt{e}$.
\end{prop}

Our results also apply to multiple kernel learning,
where in fact the exact maximizer of $\hat J_\lambda$ is efficiently available (\cref{prop:mkl-soln}).
\begin{prop} \label{thm:main-mkl-lip}
    Let $\{k_i\}_{i=1}^D$ be a fixed set of kernels,
    with $\sup_x k_i(x, x) \le K$ for all $i$.
    Then
    picking $k_\omega = \sum_{i=1}^D \omega_i k_i$
    among some set of $\omega$ with $\sum_{i=1}^D \omega_i^2 \le R_\Omega^2$
    satisfies the conditions of \cref{thm:test-power-conv}
    with $L_k = K \sqrt D$.
\end{prop}

We finally establish our results for fully-connected deep kernels;
it also applies to convolutional networks with a slightly different $R_\Omega$ (\cref{remark:convnets}).
The constants in $L_k$ are given in \cref{thm:kern-lip}.
\begin{prop} \label{thm:main-kern-lip}
    Take $k_\omega$ as in \cref{sec:DKforTST},
    with $\phi_\omega$ a fully-connected network with depth $\Lambda$ and $D$ total parameters,
    whose activations are 1-Lipschitz with $\sigma(0) = 0$ (e.g.\ ReLU).
    Suppose the operator norm of each weight matrix and $L_2$ norm of each bias vector are is at most $R_\Omega$,
    and each $x \in \X$ has $\norm x \le R_X$.
    Then $k_\omega$ meets the conditions of \cref{thm:test-power-conv}
    with dimension $D$ and 
    $L_K = \mathcal{O}\left( \Lambda R_\Omega^{\Lambda - 1} \frac{R_X + 1}{\sigma_\phi} \right)$.
\end{prop}

The dependence on $s$ in \cref{thm:test-power-conv} is somewhat unfortunate,
but the ratio structure of $J$ means that otherwise,
errors in very small variances can hurt us arbitrarily.
Even so,
``near-perfect'' kernels (with reasonably large MMD and very small variance)
will likely still be chosen as the maximizer of the regularized criterion,
even if we do not estimate the (extremely large) ratio accurately.
Likewise, near-constant kernels (with very small variance but still small $J$)
will generally have their $J$ \emph{under}estimated,
and so are unlikely to be selected
when a better kernel is available. %
The $\epsilon q$ component in \eqref{eq:deepkernel_simpleForm}
may also help avoid extremely small variances.

Given $N$ data points,
this result also gives insight into how many we should use to train the kernel
and how many to test.
With perfect optimization,
\cref{thm:opt-power-rate}
shows a bound on the asymptotic power of the test is maximized by
training on $\Theta\left( \left(N \sqrt{\log N} \right)^\frac34 \right)$ points,
and testing on the remainder.

\begin{figure*}[!t]
    \begin{center}
        \subfigure
        {\pdftooltip{\includegraphics[width=0.77\textwidth]{fig/legend_crop.pdf}}{These are baselines considered in this paper. See more details in Section~\ref{sec:exp}.}}
        \subfigure[Power vs. $N$; $d = 10$]
        {\includegraphics[width=0.26\textwidth]{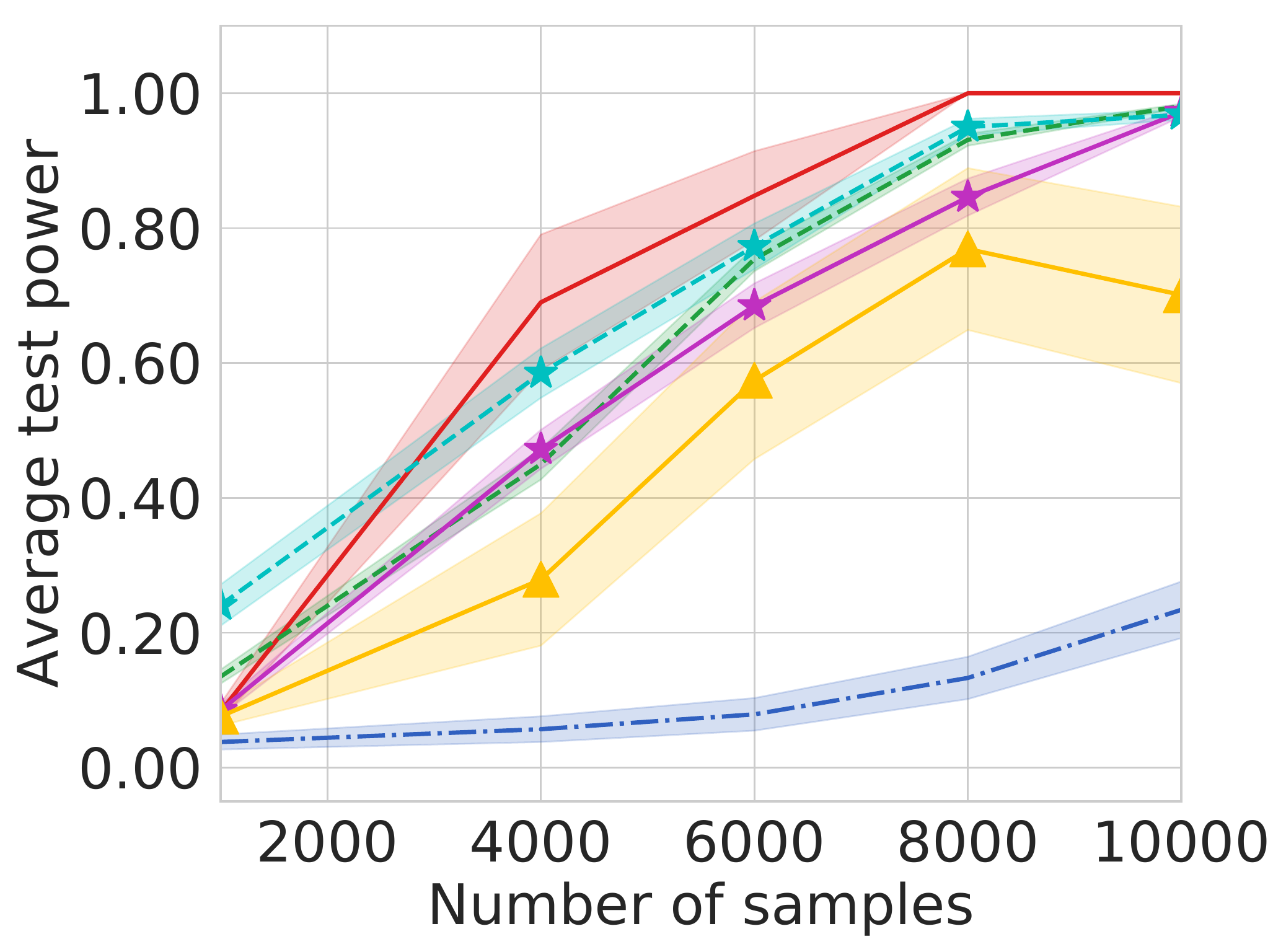}}
        \subfigure[Level vs. $N$; $d = 10$]
        {\includegraphics[width=0.24\textwidth]{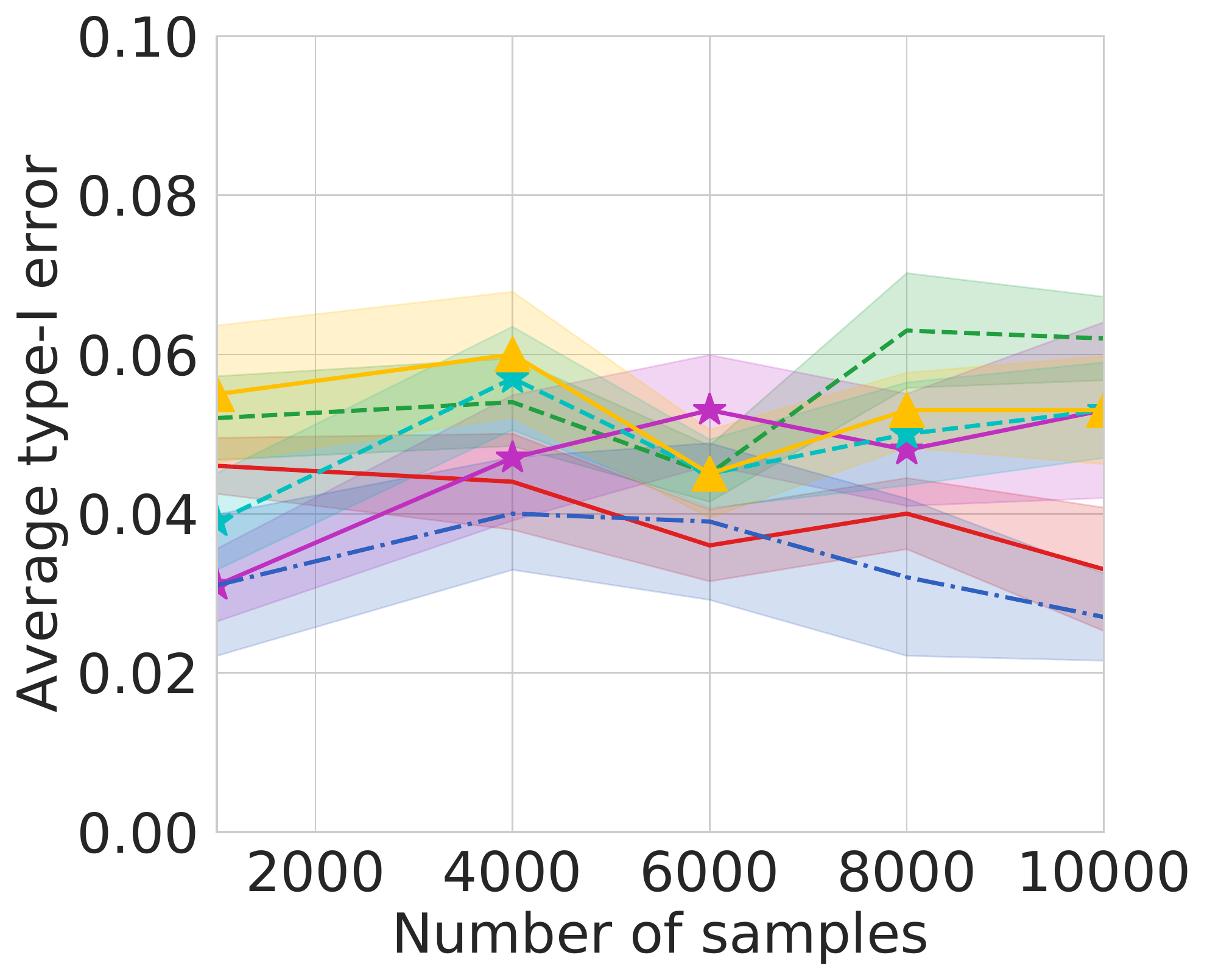}}
        \subfigure[Power vs. $d$; $N = 4\,000$]
        {\includegraphics[width=0.24\textwidth]{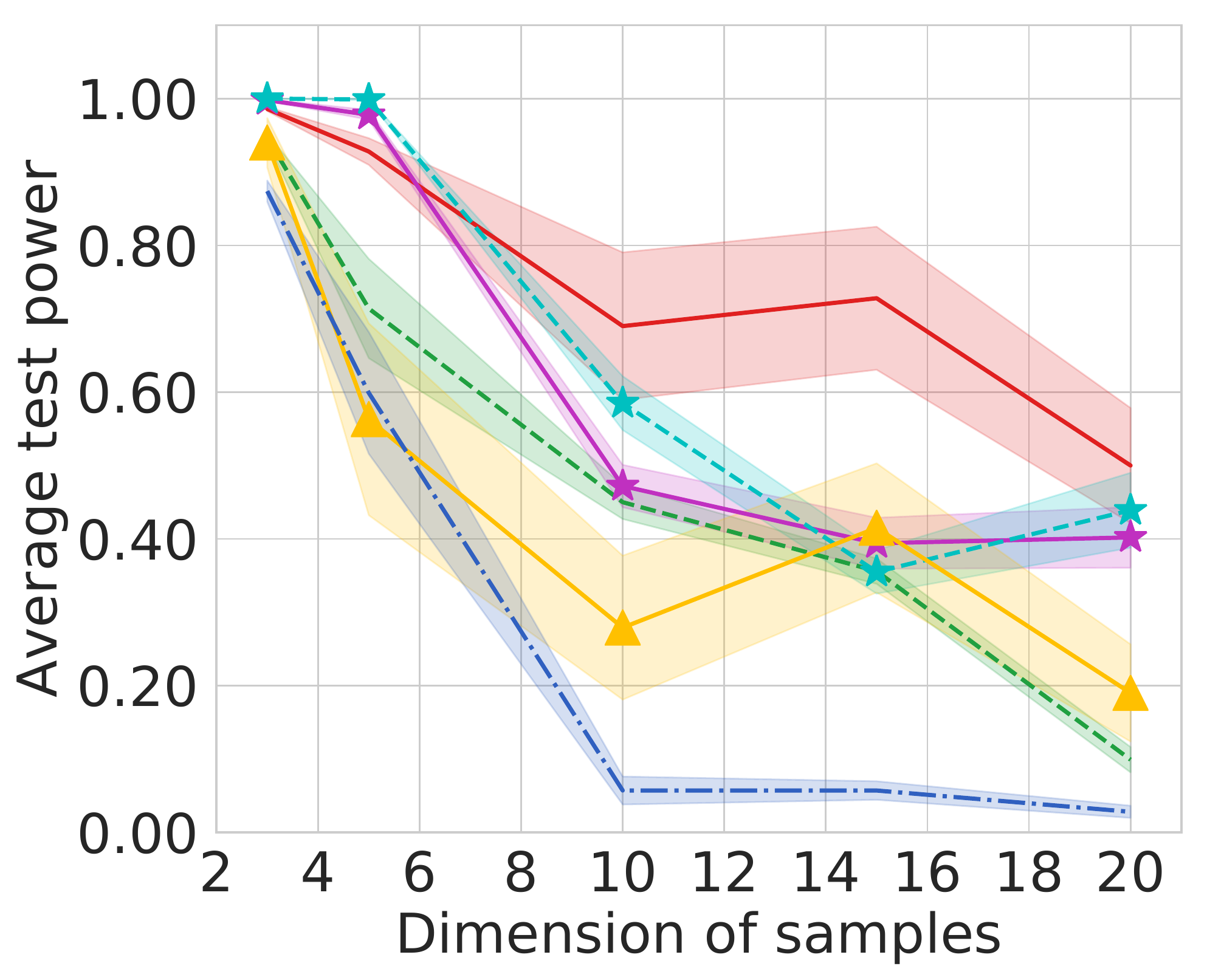}}
        \vspace{-0.3cm}
        \subfigure[Level vs. $d$; $N = 4\,000$]
        {\includegraphics[width=0.24\textwidth]{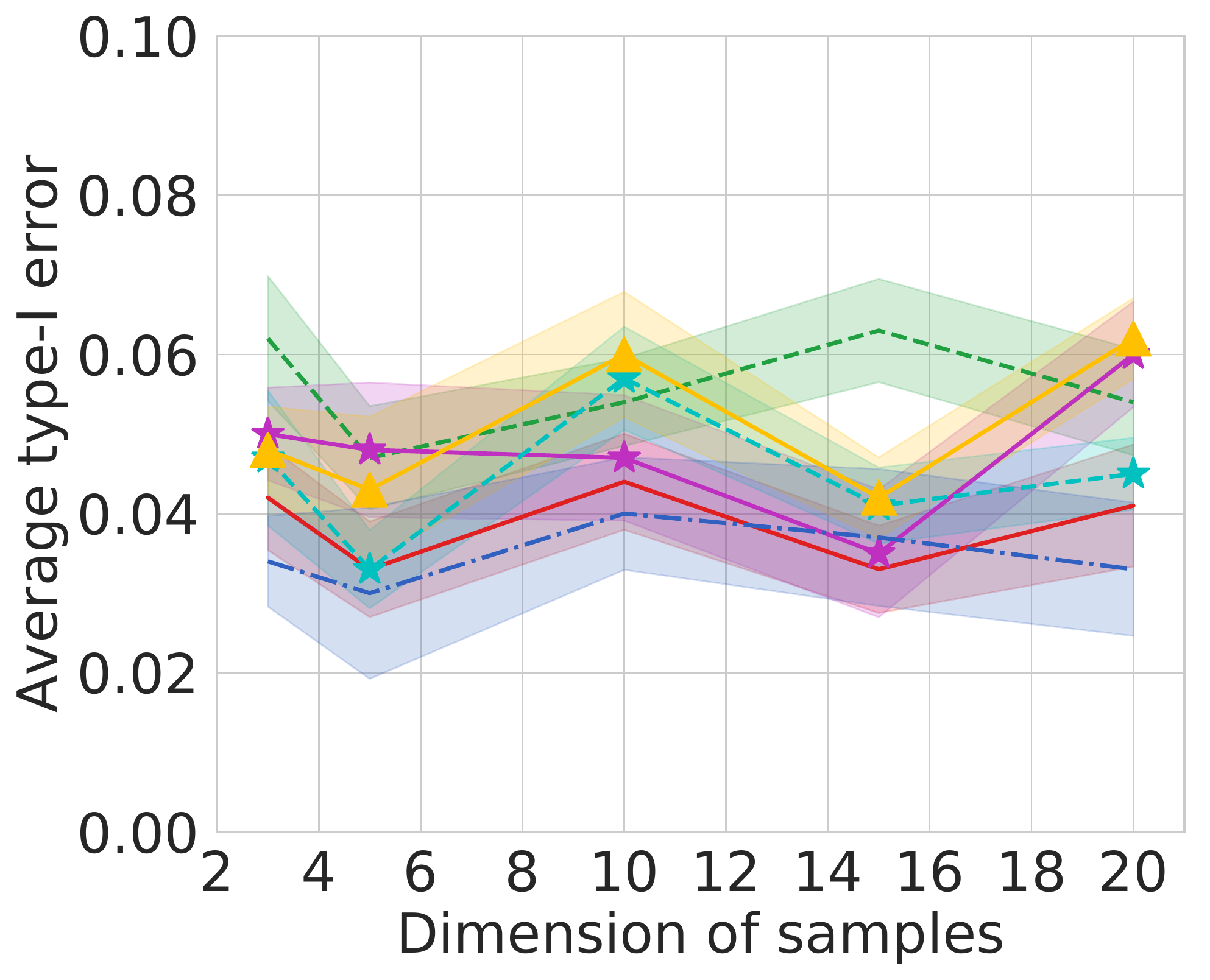}}
        \caption{Results on \emph{HDGM-S} and \emph{HDGM-D} for $\alpha=0.05$ (black line).
        Left: average test power (a) and Type I error (b) when increasing the number of samples $N$, keeping $d = 10$.
        Right: average test power (c) and Type I error (d) when increasing the dimension $d$, keeping $N = 4\,000$. Shaded regions show standard errors for the mean.}  \label{fig:HDGM_RES}
    \end{center}
    \vspace{-1em}
\end{figure*}

\section{Experimental Results}
\label{sec:exp}

\subsection{Comparison on Benchmark Datasets} \label{sec:benchmark-exp}

We compare the following tests on several datasets:
\begin{compactitem}
    \item MMD-D: \underline{MMD} with a \underline{d}eep kernel; our method described in \cref{sec:DKforTST}.
    \item MMD-O: \underline{MMD} with a Gaussian kernel whose lengthscale is \underline{o}ptimized as in \cref{sec:DKforTST}. This gives better results than standard heuristics.
    \item \underline{M}ean \underline{e}mbedding (ME): a state-of-the-art test \citep{Chwialkowski2015,Jitkrittum2016}
    based on differences in Gaussian kernel mean embeddings
    at a set of optimized points.
    \item \underline{S}mooth \underline{c}haracteristic \underline{f}unctions (SCF): a state-of-the-art test \citep{Chwialkowski2015,Jitkrittum2016}
    based on differences in Gaussian mean embeddings at a set of optimized frequencies.
    \item \underline{C}lassifier \underline{two}-\underline{s}ample \underline{t}ests, including C2STS-S \citep{Lopez:C2ST} and C2ST-L \citep{cheng:net-logits} as described in \cref{sec:c2st-relation}.
    We set the test thresholds via permutation for both.
\end{compactitem}

For synthetic datasets,
we take a single sample set for $S^{tr}_\P$ and $S^{tr}_\Q$
and learn a kernel/test locations/etc once for each method on that training set.
We then evaluate its rejection rate
on 100 new sample sets $S^{te}_\P$, $S^{te}_\Q$
from the same distribution.
For real datasets,
we select a subset of the available data for $S^{tr}_\P$ and $S^{tr}_\Q$
and train on that;
we then evaluate on 100 random subsets, disjoint from the training set, of the remaining data.
We repeat this full process 10 times,
and report the mean rejection rate of each test.
\Cref{tab:t_test_RES} shows significance tests. %
Further details
are in \cref{Asec:exp_set}.

\vspace{-1.5ex}\paragraph{\emph{Blob} dataset.}
\emph{Blob-D} is the dataset shown in \cref{fig:moti};
\emph{Blob-S} has $\Q$ also equal to the distribution shown in \cref{fig:moti}a, so that the null hypothesis holds.
Details are given in \cref{tab:synthetic datasets} (\cref{Asec:syn_intro}).

Results are shown in \cref{fig:Blob_RES}.
MMD-D and C2ST-L are the clear winners in power,
with MMD-D better in the higher-sample regime, and MMD-D is more reliable than C2STs.
\Cref{fig:Blob_RES}c shows that $J$ is higher for MMD-D than MMD-O,
in addition to the actual test power being better,
as discussed in \cref{sec:MMD_limited}.
All methods have expected Type I error rates.

\begin{table*}[ht]
\centering
  \footnotesize
  \caption{\emph{Higgs} ($\alpha=0.05$): average test power$\pm$standard error for $N$ samples. Bold represents the highest mean per row.} \label{tab:Higgs_RES1}
\vspace{1mm}
\begin{tabular}{c|cccccc}
\toprule
$N$ & ME & SCF & C2ST-S & C2ST-L &MMD-O & MMD-D \\
\midrule
\phantom{1}1$\,$000 & \mnstd{0.120}{0.007} & \mnstd{0.095}{0.022} & \mnstd{0.082}{0.015} & \mnstd{0.097}{0.014} & \mnstd{\bf 0.132}{0.005} & \mnstd{0.113}{0.013} \\
\phantom{1}2$\,$000 & \mnstd{0.165}{0.019} & \mnstd{0.130}{0.026} & \mnstd{0.183}{0.032} & \mnstd{0.232}{0.017} & \mnstd{0.291}{0.012} & \mnstd{\bf 0.304}{0.035} \\
\phantom{1}3$\,$000 & \mnstd{0.197}{0.012} & \mnstd{0.142}{0.025} & \mnstd{0.257}{0.049} & \mnstd{0.399}{0.058} & \mnstd{0.376}{0.022} & \mnstd{\bf 0.403}{0.050} \\
\phantom{1}5$\,$000 & \mnstd{0.410}{0.041} & \mnstd{0.261}{0.044} & \mnstd{0.592}{0.037} & \mnstd{0.447}{0.045} & \mnstd{0.659}{0.018} & \mnstd{\bf 0.699}{0.047} \\
\phantom{1}8$\,$000 & \mnstd{0.691}{0.067} & \mnstd{0.467}{0.038} & \mnstd{0.892}{0.029} & \mnstd{0.878}{0.020} & \mnstd{0.923}{0.013} & \mnstd{\bf 0.952}{0.024} \\
          10$\,$000 & \mnstd{0.786}{0.041} & \mnstd{0.603}{0.066} & \mnstd{0.974}{0.007} & \mnstd{0.985}{0.005} & \mnstd{\bf 1.000}{0.000} & \mnstd{\bf 1.000}{0.000} \\
\midrule
Avg. & 0.395 & 0.283 & 0.497 & 0.506 & 0.564 & {\bf 0.579} \\
\bottomrule
\end{tabular}
\vspace{-1em}
\end{table*}

\begin{table*}[ht]
\centering
  \footnotesize
  \caption{\emph{MNIST} ($\alpha = 0.05$): average test power$\pm$standard error for comparing $N$ real images to $N$ DCGAN samples. %
   } \label{tab:MNIST_RES1}
\vspace{1mm}
\begin{tabular}{c|cccccc}
\toprule
$N$ & ME & SCF & C2ST-S& C2ST-L & MMD-O & MMD-D \\
\midrule
\phantom{1}$\,$200 & \mnstd{0.414}{0.050} & \mnstd{0.107}{0.018} & \mnstd{0.193}{0.037} & \mnstd{0.234}{0.031} & \mnstd{0.188}{0.010} & \mnstd{\bf 0.555}{0.044}  \\
\phantom{1}$\,$400 & \mnstd{0.921}{0.032} & \mnstd{0.152}{0.021} & \mnstd{0.646}{0.039} & \mnstd{0.706}{0.047} & \mnstd{0.363}{0.017} & \mnstd{\bf 0.996}{0.004} \\
\phantom{1}$\,$600 & \mnstd{\bf 1.000}{0.000} & \mnstd{0.294}{0.008} & \mnstd{\bf 1.000}{0.000} & \mnstd{0.977}{0.012} & \mnstd{0.619}{0.021} & \mnstd{\bf 1.000}{0.000} \\
\phantom{1}$\,$800 & \mnstd{\bf 1.000}{0.000} & \mnstd{0.317}{0.017} & \mnstd{\bf 1.000}{0.000} & \mnstd{\bf 1.000}{0.000} & \mnstd{0.797}{0.015} & \mnstd{\bf 1.000}{0.000} \\
          1$\,$000 & \mnstd{\bf 1.000}{0.000} & \mnstd{0.346}{0.019} & \mnstd{\bf 1.000}{0.000} & \mnstd{\bf 1.000}{0.000} & \mnstd{0.894}{0.016} & \mnstd{\bf 1.000}{0.000} \\
\midrule
Avg. & 0.867 & 0.243 & 0.768 & 0.783 & 0.572 & {\bf 0.910} \\
\bottomrule
\end{tabular}
\vspace{-1em}
\end{table*}

\vspace{-1.5ex}\paragraph{High-dimensional Gaussian mixtures.}
Here we study bimodal Gaussian mixtures in increasing dimension.
Each distribution has two Gaussian components;
in \emph{HDGM-S}, $\P$ and $\Q$ are the same,
while in \emph{HDGM-D}, $\P$ and $\Q$ differ in the covariance of a single dimension pair but are otherwise the same.
Details are in \cref{tab:synthetic datasets} (\cref{Asec:syn_intro}).
We consider both increasing $N$ while keeping $d = 10$
and increasing $d$ while keeping $N = 4\,000$,
with results shown in \cref{fig:HDGM_RES}.
Again, MMD-D has generally the best test power across a range of problem settings,
with reasonable type I error.

\vspace{-1.5ex}\paragraph{\emph{Higgs} dataset \citep{Baldi_Higgs_datasets}.}
We compare the jet $\phi$-momenta distribution ($d = 4$) of the background process, $\P$, which lacks Higgs bosons,
to the corresponding distribution $\Q$ for the process that produces Higgs bosons,
following \citet{Chwialkowski2015}.
As discussed in these previous works, $\phi$-momenta carry very little discriminating information for recognizing whether Higgs bosons were produced.
We consider a series of tests with increased number of samples $N$.

We report average test power (comparing $\P$ to $\Q$) in \cref{tab:Higgs_RES1},
and average type-I error (comparing $\P$ to $\P$ or $\Q$ to $\Q$) in \cref{tab:Higgs_RES2} (\cref{sec:typeI}).
As before, MMD-D generally performs the best;
although the improvement over MMD-O here is not dramatic,
MMD-D does notably outperform C2ST.
All methods maintain reasonable Type I errors.

\vspace{-1.5ex}\paragraph{\emph{MNIST} generative model.}
The \emph{MNIST} dataset contains $70\,000$ handwritten digit images  \citep{lecun1998gradient}.
We compare true \emph{MNIST} data samples $\P$
to samples $\Q$ from a pretrained {deep convolutional generative adversarial network} (DCGAN) \citep{DCGAN_Radford}.
Samples from both distributions are shown in \cref{fig:MNIST} (in \cref{Asec:data_visual}).

We consider tests for increasing numbers of samples $N$,
and report average test power (for $\P$ to $\Q$) in \cref{tab:MNIST_RES1}
and average Type I error ($\P$ to $\P$) in \cref{tab:MNIST_RES2} (in \cref{sec:typeI}).
MMD-D substantially outperforms its competitors in test power,
with the desired Type I error.
ME also does well in this case:
it is perhaps particularly suited to this problem,
since it is capable of identifying either modes dropped by the generative model or spurious modes it inserts.

\vspace{-1.5ex}\paragraph{\emph{CIFAR-10} vs \emph{CIFAR-10.1}.}
\emph{CIFAR-10.1} \citep{recht:imagenet} is an attempt to collect a new test set for the very popular \emph{CIFAR-10} image classification dataset \citep{cifar10}.
Normally, when evaluating a supervised model,
we consider the test set an independent sample from the training distribution, ideally never-before-seen by the training algorithm.
But modern computer vision model architectures and training procedures have been developed based on repeatedly evaluating on the \emph{CIFAR-10} test set ($\P$),
so it is possible that current models themselves are dependent on $\P$.
\emph{CIFAR-10.1} ($\Q$) is an attempt at an independent sample from this distribution, collected after the models were trained, so that they are truly independent of $\Q$.
These models do obtain substantially lower accuracies on $\Q$ than on $\P$ -- but this drop is surprisingly consistent across models, which seems unlikely to be due to the expected overfitting.
The main potential explanation proposed by \citeauthor{recht:imagenet}\ is dataset shift,
but their attempt (in their Appendix C.2.8) at what amounts to a C2ST-S did not reject $\nullhyp$.\footnote{Assuming pretrained classifiers are independent of $\P$, Figure 1 of \citet{recht:imagenet} indicates that the joint (images, labels) distribution certainly differs between \emph{CIFAR-10} and \emph{CIFAR-10.1}. We test here whether the marginal image distribution differs.}
Samples from each distribution are shown in \cref{fig:CIFAR10} (\cref{Asec:data_visual}).

We train on $1\,000$ images from each dataset and test on $1\,031$, so that we use the entirety of \emph{CIFAR-10.1} each time, and average over ten repetitions.
These tests provide strong evidence (\cref{tab:cifar10_RES}) that images in the \emph{CIFAR-10.1} test set \emph{are} statistically different from the \emph{CIFAR-10} test set,
with MMD-D again strongest
and ME still performing well.

Our learned kernel also helps provide some ability to interpret the difference between $\P$ and $\Q$,
particularly if we use it for an ME test.
\Cref{sec:cifar-interp} explores this.

\begin{table}[t!]
  \centering
  \footnotesize
  \caption{\emph{CIFAR-10.1} ($\alpha=0.05$): mean rejection rates.%
  }
  \vspace{1mm}
    \begin{tabular}{llllll}
    \toprule
    ME & SCF & C2ST-S & C2ST-L & MMD-O & MMD-D \\
    \midrule
    0.588 & 0.171 & 0.452 & 0.529 & 0.316 & {\bf 0.744} \\
    \bottomrule
    \end{tabular}%
  \label{tab:cifar10_RES}%
  \vspace{-1em}
\end{table}%

\citet{recht:imagenet} also provide a new ImageNetV2 test set for the ImageNet dataset, with similar properties;
we defer this more challenging problem to future work.

\begin{table*}[!t]
  \centering
  \footnotesize
  \caption{Mean test power on \emph{Blob} ($n_b=40$), \emph{HDGM} ($N=4000,d=10$), \emph{Higgs} ($N=3000$) and \emph{MNIST} ($N=400$) for $\alpha=0.05$.  See \cref{sec:tpp-vs-ce} for the naming scheme; S+C corresponds to C2ST-S, L+C to C2ST-L, and D+J to MMD-D.
  L+M is the method proposed by \citet{Matthias:deep-test}.
  }\label{tab:CE_for_TST}
  \vspace{1mm}
    \begin{tabular}{lllllllllll}
    \toprule
    & S+C & L+C & G+C & D+C & L+M & G+M & D+M & L+J & G+J & D+J \\
    \midrule
    \emph{Blob} 
    & 0.835 & 0.942 & 0.901
    & 0.900
    & 0.851
    & 0.960
    & 0.906
    & 0.952 & 0.966 & {\bf 0.985} \\
    \emph{HDGM} 
    & 0.472 & 0.585 & 0.287
    & 0.302
    & 0.494
    & 0.223
    & 0.539
    & 0.635 & 0.604 & {\bf 0.659} \\
    \emph{Higgs}
    & 0.257 & 0.399 & 0.353
    & 0.384
    & 0.321
    & 0.254
    & 0.379
    & 0.295 & 0.364 & {\bf 0.403} \\
    \emph{MNIST}
    & 0.646 & 0.706 & 0.784
    & 0.803
    & 0.845
    & 0.680
    & 0.760
    & 0.935 & 0.976 & {\bf 0.996} \\
    \midrule
    Avg.         & 0.553 & 0.658 & 0.581
    & 0.597
    & 0.628
    & 0.529
    & 0.646
    & 0.704 & 0.727 & {\bf 0.761} \\
    \bottomrule
    \end{tabular}%
    \vspace{-1em}
\end{table*}

\begin{table}[!t]
  \centering
  \footnotesize
  \caption{Paired t-test results ($\alpha = 0.05$) for the results of \cref{sec:benchmark-exp}. For \textit{HDGM}, we fix $d=10$ (corresponding to \cref{fig:HDGM_RES}a). $\checkmark$ indicates MMD-D achieved statistically significantly higher mean test power than the other method, $\times$ that it did not.}
  \vspace{1mm}
    \begin{tabular}{l|ccccc}
    \toprule
    Dataset & \multicolumn{1}{l}{ME} & \multicolumn{1}{l}{SCF} & \multicolumn{1}{l}{C2ST-S} & \multicolumn{1}{l}{C2ST-L} & \multicolumn{1}{l}{MMD-O} \\
    \midrule
    \emph{Blob} &  \checkmark  &  \checkmark  &  \checkmark  & $\times$ & $\times$ \\
    \emph{HDGM} &  \checkmark  &  \checkmark  &  \checkmark & \checkmark & \checkmark \\
    \emph{Higgs} &  \checkmark  & \checkmark   &  \checkmark  & $\times$ & $\times$ \\
    \emph{MNIST} &  \checkmark  &  \checkmark  &  \checkmark  & \checkmark &  \checkmark\\
    \bottomrule
    \end{tabular}%
  \label{tab:t_test_RES}%
  \vspace{-2em}
\end{table}%

\subsection{Ablation Study} \label{sec:tpp-vs-ce}

We now study in more detail the difference between MMD-D and closely related methods.
Recall from \cref{sec:c2st-relation} that there are two main differences between MMD-D and C2STs:
first,
using a ``full'' kernel \eqref{eq:deepkernel_simpleForm}
rather than the sign-based kernel \eqref{eq:sign-kernel}
or the intermediate linear kernel \eqref{eq:lin-kernel}.
Second, training to maximize $\hat J_\lambda$ \eqref{eq:tpp-hat}
rather than a cross-entry surrogate.
MMD-D uses a full kernel \eqref{eq:deepkernel_simpleForm} trained for test power;
C2ST-S effectively uses the sign kernel \eqref{eq:sign-kernel} trained for cross entropy.

In this section, we consider the performance of several intermediate models empirically,
demonstrating that both factors help in testing.
All are based on the same feature extraction architecture $\phi_\omega$;
some models add a classification layer with new parameters $w$ and $b$,
\[
    f_\omega(x) = w\tp \phi_\omega(x) + b
,\]
which is treated as outputting classification logits.
The model variants we consider are
\begin{compactdesc}
\item[S] A kernel $\mathbbm{1}(f_\omega(x) > 0) \mathbbm{1}(f_\omega(y) > 0)$; corresponds to a test statistic of the accuracy of $f$ (\cref{thm:c2st-equiv}).
\item[L] A kernel $f_\omega(x) f_\omega(y)$; corresponds to a test statistic comparing the mean value of $f$ (\cref{thm:c2st-l-equiv}).
\item[G] A Gaussian kernel $\kappa(\phi_\omega(x),\phi_\omega(y))$.
\item[D] The deep kernel \eqref{eq:deepkernel_simpleForm} based on $\phi_\omega$.
\end{compactdesc}
We combine these model variants with a suffix describing the optimization objective:
\begin{compactdesc}
    \item[J] Choose $\omega$, including possibly $w$ and $b$, to optimize the approximate test power \eqref{eq:tpp-hat}.
    \item[M] Choose $\omega$, including possibly $w$ and $b$, to maximize the value of the empirical MMD between two samples.\footnote{If a deep kernel is unbounded, directly maximizing MMD will make optimized parameters of $\phi_\omega$ be infinite. Thus, for L+M, we consider a normalized linear deep kernel: $\textnormal{tanh}(f_\omega(x)/\|S\|_\textnormal{F})\textnormal{tanh}(f_\omega(y)/\|S\|_\textnormal{F})$, where $S = [S_\P;S_\Q]$ and $\|\cdot\|_F$ is the Frobenius norm.}
    \item[C] Choose $\omega$, including $w$ and $b$, to optimize cross-entropy using the classifier that specifies the probability of $x$ belonging to $\P$ as $1 / \left( 1 + \exp(- f_\omega(x) ) \right)$.\footnote{G+C and D+C take the fixed $\phi_\omega$ embeddings, then find the optimal lengthscale/etc by optimizing $\hat J_\lambda$.}
\end{compactdesc}

\Cref{tab:CE_for_TST} presents results for all of these methods
(except for S+J, which is non-differentiable and hence difficult to optimize).
Performance generally improves
as we move from S to L to G to D,
and from C to J.

\subsection{Architecture design of deep kernels}
For \emph{Blob}, \emph{HDGM} and \emph{Higgs}, $\phi_\omega$ is a five-layer fully-connected neural network, with softplus activations. the number of neurons in hidden and output layers of $\phi_\omega$ are set to $50$ for \emph{Blob}, $3 d$ for \emph{HDGM} and $20$ for \emph{Higgs}, where $d$ is the dimension of samples.
in general, we expect similar fully-connected networks,
to be reasonable choices for datasets where strong structural assumptions are not known,
perhaps with $3 d$ as a baseline width for datasets of at least moderate dimension.

For \emph{MNIST} and \emph{CIFAR}, $\phi_\omega$ is a \emph{convolutional neural network} (CNN) that contains four convolutional layers and one fully-connected layer.
The structure of the CNN follows the structure of the feature extractor in the DCGAN's discriminator \citep{DCGAN_Radford} (see \cref{fig:MMDDK_phi,fig:MMD_CIFAR_F} for the structure of $\phi_\omega$ in MMD-D, and \cref{fig:C2ST_F,fig:C2ST_CIFAR_F} for the structure of classifier $F$ in C2ST-S and C2ST-L).
In general, we expect GAN discriminator architectures to work well for image datasets,
as the problem is closely related.

\section{Conclusions}
The test power of MMD is limited by simple kernels (e.g., Gaussian kernel or other translation-invariant kernels) when facing complex-structured distributions,
but we can avoid this problem with richer \emph{deep kernels}, which is no longer translation-invariant.
We show that optimizing the parameters of these kernels to maximize the test power,
as proposed by \citet{sutherland:mmd-opt},
outperforms state-of-the-art alternatives
even when considering large, deep kernels with hundreds of thousands of parameters,
rather than the simple shallow kernels they considered.
We provide theoretical guarantees that this process is reasonable to conduct on finite samples, and asymptotically selects the most powerful kernel.
We also give deeper insight into the relationship between this approach and classifier two-sample tests \citep{Lopez:C2ST},
explaining why this approach outperforms that one.

We thus recommend practitioners to use optimized deep kernel methods when they wish to check if two distributions are the same, rather than indirectly training a classifier.

\ifdefined\isaccepted
\section*{Acknowledgements}
This work was supported by the Australian Research Council under FL190100149 and DP170101632, and
by the Gatsby Charitable Foundation.
FL, JL and GZ gratefully acknowledge the support of the NVIDIA Corporation with the donation of two NVIDIA TITAN V GPUs for this work. FL also acknowledges the support from UTS-FEIT and UTS-AAII.
DJS would like to thank Aram Ebtekar, Ameya Velingker, and Siddhartha Jain for productive discussions.
\fi

\bibliography{mybib}
\bibliographystyle{icml2020}

\clearpage
\onecolumn
\appendix

\section{Theoretical analysis}\label{Asec:proof}
\Cref{sec:proof:main} proves the main results under some assumptions about the kernel parameterization,
using intermediate results about uniform convergence of our estimators in \cref{sec:proof:unif-conv}.
\Cref{sec:proof:kernel-props} then shows that these assumptions hold for different settings of kernel learning.

\subsection{Preliminaries}
Given a kernel $k_\omega$ and sample sets $\{X_i\}_{i=1}^n \sim \P^n$, $\{Y_i\}_{i=1}^n \sim \Q^n$, define the $n \times n$ matrix
\[
    H_{ij}^{(\omega)} = k_\omega(X_i, X_j) + k_\omega(Y_i, Y_j) - k_\omega(X_i, Y_j) - k_\omega(X_j, Y_i)
;\]
we will often omit $\omega$ when it is clear from context.
The $U$-statistic estimator of the squared MMD \eqref{eq:MMD_U_compute} is
\[
    \hat\eta_\omega = \frac{1}{n (n-1)} \sum_{i \ne j} H_{ij}
.\]
The squared MMD is $\eta_\omega = \E[H_{12}]$.
The variance of $\hat\eta_\omega$ is given by \cref{thm:var-decomp}.

\begin{lemma} \label{thm:var-decomp}
    For a fixed kernel $k_\omega$ and random sample sets $\{X_i\}_{i=1}^n$, $\{Y_i\}_{i=1}^n$,
    we have
    \begin{equation} \label{eq:var-decomp}
        \Var[\hat\eta_\omega]
        = \frac{4(n-2)}{n (n-1)} \xi_1^{(\omega)} + \frac{2}{n (n-1)} \xi_2^{(\omega)}
        = \frac4n \xi_1^{(\omega)} + \frac{2 \xi_2^{(\omega)} - 4 \xi_1^{(\omega)}}{n (n-1)}
    ,\end{equation}
    where
    \[
        \xi_1^{(\omega)} = \E\left[H_{12}^{(\omega)} H_{13}^{(\omega)}\right]
              - \E\left[H_{12}^{(\omega)}\right]^2
        ,\qquad
        \xi_2^{(\omega)} = \E\left[\left(H_{12}^{(\omega)}\right)^2\right]
              - \E\left[H_{12}^{(\omega)}\right]^2
    .\]
    Thus as $n \to \infty$,
    \[
        n \Var[\hat\eta_\omega] \to 4 \xi_1^{(\omega)} =: \sigma_\omega^2
    .\]
\end{lemma}
\begin{proof}
Let $U$ denote the pair $(X, Y)$,
and $h_\omega(U, U') = k_\omega(X, X') + k_\omega(Y, Y') - k_\omega(X, Y') - k_\omega(X', Y)$,
so that $H_{ij}^{(\omega)} = h_\omega(U_i, U_j)$.
Via Lemma A in Section 5.2.1 of \citet{serfling},
we know that \eqref{eq:var-decomp} holds with
\begin{align*}
    \xi_1^{(\omega)}
  &= \Var_U\left[ \E_{U'}\left[ h_\omega(U, U') \right]\right]
\\&= \E_U\left[ \E_{U'}\left[ h_\omega(U, U') \right] \E_{U''}\left[ h_\omega(U, U'') \right] \right]
   - \E_U\left[ \E_{U'}[ h_\omega(U, U') ] \right]^2
\\&= \E[ H_{12}^{(\omega)} H_{13}^{(\omega)} ] - \E[ H_{12}^{(\omega)} ]^2
\end{align*}
and
\[
    \xi_2
   = \Var_{U, U'}\left[ h_\omega(U, U') \right]
   = \E\left[ \left(H_{12}^{(\omega)}\right)^2 \right]
   - \E\left[H_{12}^{(\omega)}\right]^2
.\qedhere\]
\end{proof}

We use a $V$-statistic estimator \eqref{eq:estimate_sigma_H1} for $\sigma_\omega^2$:
\begin{align*}
     \hat\sigma_\omega^2
  &= 4 \left(
        \frac{1}{n} \sum_{i=1}^n
          \left( \frac1n \sum_{j=1}^n H_{ij}^{(\omega)} \right)^2
      - \left( \frac{1}{n^2} \sum_{i=1}^n \sum_{j=1}^n H_{ij}^{(\omega)} \right)^2
     \right)
.\end{align*}
As a $V$-statistic,
$\hat\sigma_\omega^2$ is biased.
In fact, \citet{sutherland:mmd-opt} and \citet{unbiased-var-ests} provide an unbiased estimator of $\Var[\hat\eta_\omega]$ -- including the terms of order $\frac{1}{n (n-1)}$.
Although this estimator takes the same quadratic time to compute as \eqref{eq:estimate_sigma_H1},
it contains many more terms, which are cumbersome both for implementation and for analysis.
\eqref{eq:estimate_sigma_H1} is also marginally more convenient in that it is always at least nonnegative.
As we show in \cref{thm:var-est-bias},
the amount of bias is negligible as $n$ increases.
In practice, we expect the difference to be unimportant
-- or the $V$-statistic may in fact be beneficial, since underestimating $\sigma^2$ harms the estimate of ${\eta}/{\sigma^2}$ more than overestimating it does.

Similarly, although we use the $U$-statistic estimator \eqref{eq:MMD_U_compute},
it would be very similar to use
the biased estimator $n^{-2} \sum_{ij} H_{ij}$,
or the minimum variance unbiased estimator $n^{-1} (n-1)^{-1} \sum_{i \ne j} (k(X_i, X_j) + k(Y_i, Y_j)) - 2 n^{-2} \sum_{ij} k(X_i, Y_J)$.
Showing comparable concentration behavior to \cref{thm:mmd-conv} is trivially different,
and in fact it is also not difficult to show $\sigma_\omega^2$ is the same for all three estimators (up to lower-order terms).

\subsection{Main results} \label{sec:proof:main}
We will require the following assumptions.
These are fairly agnostic as to the kernel form;
\cref{sec:proof:deep-kernels} shows that these assumptions hold
(and gives the constants)
for the kernels \eqref{eq:deepkernel_simpleForm} we use in the paper.
\begin{assumplist}
  \item \label{assump:k-bounded}
    The kernels $k_\omega$ are uniformly bounded:
    \[
      \sup_{\omega \in \Omega} \sup_{x \in \X} k_\omega(x, x) \le \nu
    .\]
    For the kernels we use in practice, $\nu = 1$.

  \item \label{assump:omega-bounded}
    The possible kernel parameters $\omega$
    lie in a Banach space of dimension $D$.
    Furthermore, the set of possible kernel parameters $\Omega$
    is bounded by $R_\omega$,
    $\Omega \subseteq \left\{ \omega \mid \norm\omega \le R_\Omega \right\}$.

    \Cref{sec:proof:deep-kernels} builds this space and its norm for the kernels we use in the paper.

  \item \label{assump:k-lipschitz}
    The kernel parameterization is Lipschitz:
    for all $x, y \in \X$
    and $\omega, \omega' \in \Omega$,
    \[
        \abs{k_\omega(x, y) - k_{\omega'}(x, y)} \le L_k \norm{\omega - \omega'}
    .\]
    \cref{thm:kern-lip} in \cref{sec:proof:deep-kernels} gives an expression for $L_k$ for the kernels we use in the paper.
\end{assumplist}

We will first show the main results under these general assumptions,
using uniform convergence results shown in \cref{sec:proof:unif-conv},
then show \cref{assump:omega-bounded,assump:k-lipschitz} for particular kernels in \cref{sec:proof:deep-kernels}.

\begin{theorem} \label{thm:ratio-conv}
Under \cref{assump:omega-bounded,assump:k-lipschitz,assump:k-bounded},
let $\bar\Omega_s \subseteq \Omega$ be the set of kernel parameters for which $\sigma_\omega^2 \ge s^2$,
and assume $\nu \ge 1$.
Take $\lambda = n^{-1/3}$.
Then, with probability at least $1 - \delta$,
\[
    \sup_{\omega \in \bar\Omega_s}
    \Bigg\lvert
        \frac{\hat\eta_\omega}{\hat\sigma_{\omega,\lambda}}
      - \frac{\eta_\omega}{\sigma_\omega}
    \Bigg\rvert
    \le
        \frac{2 \nu}{s^2 n^{1/3}} \left(
          \frac{1}{s}
        + \frac{2304 \nu^2}{\sqrt n}
        + \left[
            \frac{4 s}{n^{1/6}}
          + 1024 \nu
          \right]
          \left[
            L_k
          + \sqrt{2 \log\frac2\delta + 2 D \log\left( 4 R_\Omega \sqrt n \right)}
          \right]
        \right)
,\]
and thus, treating $\nu$ as a constant,
\[
    \sup_{\omega \in \bar\Omega_s} \abs{
        \frac{\hat\eta_\omega}{\hat\sigma_{\omega,\lambda}}
      - \frac{\eta_\omega}{\sigma_\omega}
    }
    = \tilde\bigO_P\left(\frac{1}{s^2 n^{1/3}} \left[ \frac1s + L_k + \sqrt{D} \right]
    \right)
.\]

\end{theorem}
\begin{proof}
Let $\sigma_{\omega,\lambda}^2 := \sigma_\omega^2 + \lambda$.
Using $\abs{\hat\eta_\omega} \le 4 \nu$,
we begin by decomposing
\begin{align*}
        \sup_{\omega \in \bar\Omega_s} \abs{
            \frac{\hat\eta_\omega}{\hat\sigma_{\omega,\lambda}}
          - \frac{\eta_\omega}{\sigma_\omega}
        }
  &\le  \sup_{\omega \in \bar\Omega_s}
        \abs{
            \frac{\hat\eta_\omega}{\hat\sigma_{\omega,\lambda}}
          - \frac{\hat\eta_\omega}{\sigma_{\omega,\lambda}}
        }
        + \sup_{\omega \in \bar\Omega_s}
        \abs{
            \frac{\hat\eta_\omega}{\sigma_{\omega,\lambda}}
          - \frac{\hat\eta_\omega}{\sigma_{\omega}}
        }
        + \sup_{\omega \in \bar\Omega_s}
        \abs{
            \frac{\hat\eta_\omega}{\sigma_{\omega}}
          - \frac{\eta_\omega}{\sigma_{\omega}}
        }
\\&  =
          \sup_{\omega \in \bar\Omega_s}
          \abs{\hat\eta_\omega}
          \frac{1}{\hat\sigma_{\omega,\lambda}}
          \frac{1}{\sigma_{\omega,\lambda}}
          \frac{
            \abs{\hat\sigma_{\omega,\lambda}^2 - \sigma_{\omega,\lambda}^2}
          }{\hat\sigma_{\omega,\lambda} + \sigma_{\omega,\lambda}}
        + \sup_{\omega \in \bar\Omega_s}
          \abs{\hat\eta_\omega}
          \frac{1}{\sigma_{\omega,\lambda}}
          \frac{1}{\sigma_{\omega}}
          \frac{
            \abs{\sigma_{\omega,\lambda}^2 - \sigma_\omega^2}
          }{\sigma_{\omega,\lambda} + \sigma_\omega}
        + \sup_{\omega \in \bar\Omega_s}
          \frac{1}{\sigma_\omega}
          \abs{ \hat\eta_\omega - \eta_\omega }
\\&\le
          \sup_{\omega \in \bar\Omega_s}
          \frac{4 \nu}{\sqrt\lambda \; s \; (s + \sqrt\lambda)}
          \abs{\hat\sigma_\omega^2 - \sigma_\omega^2}
        +
          \frac{4 \nu \lambda}{\sqrt{s^2 + \lambda} \; s \; \left(\sqrt{s^2 + \lambda} + s \right)}
        +
          \sup_{\omega \in \bar\Omega_s}
          \frac{1}{s} \abs{ \hat\eta_\omega - \eta_\omega }
\\&\le
          \frac{4 \nu}{s^2 \sqrt{\lambda}}
          \sup_{\omega \in \Omega} \abs{\hat\sigma_\omega^2 - \sigma_\omega^2}
        + \frac{2 \nu}{s^3} \lambda
        + \frac{1}{s} \sup_{\omega \in \Omega} \abs{ \hat\eta_\omega - \eta_\omega }
.\end{align*}
\Cref{thm:mmd-conv,thm:var-conv} show uniform convergence of $\hat\eta_\omega$ and $\hat\sigma_\omega^2$, respectively.
Thus, with probability at least $1 - \delta$,
the error is at most
\begin{equation*}
        \frac{2 \nu}{s^3} \lambda
      + \left[ \frac{8 \nu}{s \sqrt{n}} + \frac{1792 \nu}{\sqrt n s^2 \sqrt\lambda} \right]
        \sqrt{2 \log\frac2\delta + 2 D \log\left( 4 R_\Omega \sqrt n \right)}
      + \left[ \frac{8}{s \sqrt{n}} + \frac{2048 \nu^2}{\sqrt n s^2 \sqrt \lambda} \right] L_k
      + \frac{4608 \nu^3}{s^2 n \sqrt\lambda}
.\end{equation*}
Taking $\lambda = n^{-1/3}$ gives
\begin{equation*}
        \frac{2 \nu}{s^3 n^{1/3}}
      + \left[ \frac{8 \nu}{s \sqrt{n}} + \frac{1792 \nu}{s^2 n^{1/3}} \right]
        \sqrt{2 \log\frac2\delta + 2 D \log\left( 4 R_\Omega \sqrt n \right)}
      + \left[ \frac{8}{s \sqrt{n}} + \frac{2048 \nu^2}{s^2 n^{1/3}} \right] L_k
      + \frac{4608 \nu^3}{s^2 n^{5/6}}
.\end{equation*}
Using $1 \le \nu$, $1792 < 2048$,
we can get the slightly simpler upper bound
\begin{equation*}
        \frac{2 \nu}{s^3 n^{1/3}}
      + \left[ \frac{8 \nu}{s \sqrt{n}} + \frac{2048 \nu^2}{s^2 n^{1/3}} \right]
        \left[ L_k + \sqrt{2 \log\frac2\delta + 2 D \log\left( 4 R_\Omega \sqrt n \right)} \right]
      + \frac{4608 \nu^3}{s^2 n^{5/6}}
.\qedhere\end{equation*}
\end{proof}
It is worth noting that, if we are particularly concerned about the $s$ dependence,
we can make some slightly different choices in the decomposition to improve the dependence on $s$ while worsening the rate with $n$.

\begin{corollary} \label{thm:param-conv}
In the setup of \cref{thm:ratio-conv},
additionally assume that there is a unique population maximizer $\omega^*$ of $J$ from \eqref{eq:tpp},
i.e.\ for each $t > 0$ we have
\[
    \sup_{\omega \in \bar\Omega_s : \norm{\omega - \omega^*} \ge t} J(\P, \Q; k_\omega) < J(\P, \Q; k_{\omega^*})
.\]
For each $n$,
let $S_\P^{(n)}$ and $S_\Q^{(n)}$ be sequences of sample sets of size $n$,
let $\hat J_n(\omega)$ denote $J_{\lambda = n^{-1/3}}(S_\P^{(n)}, S_\Q^{(n)}; k_\omega)$,
and take $\hat\omega^*_n$ to be a maximizer of $\hat J_n(\omega)$.\footnote{In fact, it suffices for the $\hat\omega^*_n$ to only approximately maximize $\hat J_n$, as long as their suboptimality is $o_P(1)$.}{}
Then $\hat\omega^*_n$ converges in probability to $\omega^*$.
\end{corollary}
\begin{proof}
By \cref{thm:ratio-conv},
$
\sup_{\omega \in \bar\Omega_s} \abs{\hat J_{n}(\omega) - J(\omega)} \stackrel{P}{\to} 0
$.
Then the result follows by Theorem 5.7 of \citet{van2000asymptotic}.
\end{proof}

\begin{corollary} \label{thm:opt-power}
    In the setup of \cref{thm:ratio-conv},
    suppose we use $n$ sample points to select a kernel $\hat\omega_n \in \argmax_{\omega \in \bar\Omega_s} \hat J_\lambda(\omega)$
    and $m$ sample points to run a test of level $\alpha$.
    Let $r_{\hat\omega_n}^{(m)}$ denote the rejection threshold for a test with that kernel of size $m$.
    Define
    $J^* := \sup_{\omega \in \bar\Omega_s} J(\omega)$,
    and constants $C$, $C'$, $C''$, $N_0$ depending on $\nu$, $L_k$, $D$, $R_\Omega$ and $s$.
    For any $n \ge N_0$,
    with probability at least $1 - \delta$,
    this test procedure has power
    \[
           \Pr\left( m \, \hat\eta_{\hat\omega_n} > r^{(m)}_{\hat\omega_n} \right)
       \ge \Phi\left(
             \sqrt{m} J^*
           - C \frac{\sqrt m}{n^{\frac13}} \sqrt{\log\frac{n}{\delta}}
           - C' \sqrt{\log\frac1\alpha}
           \right)
         - \frac{C''}{\sqrt m}
    .\]
\end{corollary}
\begin{proof}
    Let $\hat\omega_n \in \argmax_{\omega \in \bar\Omega_s} \hat J_\lambda(\omega)$.
    By \cref{thm:ratio-conv},
    there are some $N_0, C$ depending on $\nu$, $L_k$, $D$, $R_\Omega$, and $s$
    such that as long as $n \ge N_0$,
    with probability at least $1 - \delta$ it holds that
    \[
        \sup_{\omega \in \bar\Omega_s} \abs{J_\lambda(\omega) - J(\omega)}
        \le \tfrac12 C n^{-\frac13} \sqrt{\log\frac{n}{\delta}}
        =: \epsilon_n
    .\]
    Assume for the remainder of this proof that this event holds.
    Letting $\omega^* \in \argmax J(\omega)$,
    we know because $\hat\omega_n$ maximizes $\hat J_\lambda$ that $\hat J_\lambda(\hat\omega_n) \ge \hat J_\lambda(\omega^*)$.
    Using uniform convergence twice,
    \[
        J(\hat\omega_n)
        \ge \hat J_\lambda(\hat\omega_n) - \epsilon_n
        \ge \hat J_\lambda(\omega^*) - \epsilon_n
        \ge \left( J(\omega^*) - \epsilon_n \right) - \epsilon_n
        = J^* - 2 \epsilon_n
    .\]

    Now, although
    \cref{prop:asymptotics} establishes that $r_\omega^{(m)} \to r_\omega$ 
    and it is even known \citep[Theorem 5]{Korolyuk1988} that
    $\abs*{r_\omega^{(m)} - r_\omega}$ is $o(1 / \sqrt m)$,
    the constant in that convergence will depend on the choice of $\omega$ in an unknown way.
    It's thus simpler to use the very loose but uniform (McDiarmid-based) bound
    given by Corollary 11 of \citet{Gretton2012},
    which implies $r_\omega^{(m)} \le 4 \nu \sqrt{\log(\alpha^{-1}) m}$
    no matter the choice of $\omega$.

    We will now need a more precise characterization of the power than that provided by the
    central limit theorem of \cref{prop:asymptotics}.
    \Citet{Callaert:berry-esseen-ustat} provide such a result, a Berry-Esseen bound on $U$-statistic convergence:
    there is some absolute constant $C'_\mathit{BS} = 2^3 4^3 C_\mathit{BS}$ such that
    \[
        \sup_t \abs{
            \Pr_{\althyp}\left( \sqrt m \frac{\hat\eta_\omega - \eta_\omega}{\sigma_\omega^2} \le t \right)
            - \Phi(t)
        }
        \le \frac{C'_\mathit{BS} \E \abs{H_{12}}^3}{(\sigma_\omega / 2)^3 \sqrt m}
        \le \frac{C_\mathit{BS} \nu^3}{\sigma_\omega^3 \sqrt m}
    .\]
    Letting $r_\omega^{(m)}$ be the appropriate rejection threshold for $k_\omega$ with $m$ samples,
    the power of a test with kernel $k_\omega$ is
    \begin{align*}
           \Pr\left( m \hat\eta_\omega > r_\omega^{(m)} \right)
      &  = \Pr\left( \sqrt m \frac{\hat \eta_\omega - \eta_\omega}{\sigma_\omega} > \frac{r_\omega^{(m)}}{\sqrt{m} \sigma_\omega} - \sqrt{m} \frac{\eta_\omega}{\sigma_\omega} \right)
    \\&\ge \Phi\left( \sqrt{m} J(\omega) - \frac{r_\omega^{(m)}}{\sqrt m \sigma_\omega} \right)
         - \frac{C_\mathit{BS} \nu^3}{\sigma_\omega^3 \sqrt m}
    \\&\ge \Phi\left( \sqrt{m} J(\omega) - \frac{r_\omega^{(m)}}{s \sqrt m} \right)
         - \frac{C''}{\sqrt m}
    ,\end{align*}
    using a new constant $C'' := C_\mathit{BS} \nu^3 / s^3$.
    Combining the previous results on $J(\hat\omega_n)$ and $r_{\hat\omega_n}^{(m)}$
    yields the claim.
\end{proof}

\begin{corollary} \label{thm:opt-power-rate}
    In the setup of \cref{thm:opt-power},
    suppose we are given $N$ data points to divide between
    $n$ training points
    and $m = N - n$ testing points,
    and $\delta < 0.22$ is fixed.
    Ignoring the Berry-Esseen convergence term outside of $\Phi$,
    the asymptotic power upper bound
    \[
       \Phi\left(
             \sqrt{m} J^*
           - C \frac{\sqrt m}{n^{\frac13}} \sqrt{\log\frac{n}{\delta}}
           - C' \sqrt{\log\frac1\alpha}
           \right)
    \]
    is maximized only when,
    as other quantities remain constant,
    we pick $n$ to satisfy
    \[
        \lim_{N \to \infty}
        \frac{n}{
            \left( \frac{C}{\sqrt{3} J^*} N \sqrt{\log N} \right)^{\frac34}
        } = 1
    .\]
\end{corollary}
\begin{proof}
    Because the $C'$ term is constant,
    we wish to choose
    \begin{align*}
         \argmax_{0 < n < N} \frac{J^*}{C} \sqrt{N - n} - \frac{\sqrt{N - n}}{n^{\frac13}} \sqrt{\log\frac{n}{\delta}}
    .\end{align*}
    Clearly neither endpoint is optimal.
    Relaxing $n$ to be real-valued,
    the optimum must be achieved at a stationary point, where
    \[
        \frac{-J^*}{2 C \sqrt{N - n}}
        + \frac{\sqrt{\log \frac n \delta}}{2 \sqrt{N - n} \, n^{\frac13}}
        + \frac13 \sqrt{N - n}\, n^{-\frac43} \sqrt{\log \frac n \delta}
        - \frac12 \sqrt{N - n}\, n^{-\frac43} \left(\log \frac n \delta\right)^{-\frac12}
        = 0
    .\]
    Multiplying by $2 \sqrt{N - n} \, n^\frac43 \sqrt{\log \frac n \delta}$ and rearranging,
    we get that a stationary point is achieved exactly when
    \[
        \underbrace{
            \frac13 \left[ n + 2 N \right] \log \frac n \delta
            + n
        }_{D}
        =
        \underbrace{
            \frac{J^*}{C} n^\frac43 \sqrt{\log\frac n \delta}
          + N
        }_{E}
    .\]

    Now write, without loss of generality, $n = \left( A_N N \sqrt{\log N} \right)^{\frac34}$,
    and so
    \begin{gather*}
        D =
            \frac13 \Bigg[
                A_N^\frac34 N^\frac34 (\log N)^\frac38
                + 2 N
            \Bigg]
            \Bigg[
                \underbrace{\frac34 \log A_N + \frac34 \log N + \frac38 \log\log N}_{\log n}
                + \log\frac1{\delta}
            \Bigg]
            + A_N^\frac34 N^\frac34 (\log N)^\frac38
        \\
        E
        = \frac{J^*}{C} A_N N \sqrt{\log N}
          \sqrt{
            \underbrace{\frac34 \log A_N + \frac34 \log N + \frac38 \log\log N}_{\log n}
            + \log\frac1{\delta}}
        + N 
    .\end{gather*}
    We will show that $D - E \to 0$ requires $A_N \to C / (\sqrt{3} J^*)$,
    implying the result.

    We first suppose $A_N = \omega(1)$,
    further breaking into cases which result in different terms inside $D$ and $E$ becoming dominant:
    \begin{alignat*}{3}
        & \text{If $A_N = \Omega(N)$,}
            \quad &&
        D = \Theta\left( A_N^\frac34 N^\frac34 (\log N)^\frac38 \log A_N \right)
            , \quad &&
        E = \Theta\left( A_N N \sqrt{\log(N) \log(A_N)} \right)
            .\\
        & \text{If $A_N = \Omega\left( \frac{N^\frac13}{\sqrt{\log N}} \right)$, $ A_N = o(N)$,}
            \quad &&
        D = \Theta\left( A_N^\frac34 N^\frac34 (\log N)^\frac38 \log N \right)
            , \quad &&
        E = \Theta\left( A_N N \log N \right)
            .\\
        & \text{If $A_N = \omega(1)$, $A_N = o\left( \frac{N^\frac13}{\sqrt{\log N}} \right)$,}
            \quad &&
        D = \Theta\left( N \log N \right)
            , \quad &&
        E = \Theta\left( A_N N \log N \right)
            .
    \end{alignat*}
    In each case, $E = \omega(D)$ and so $D - E \to -\infty$,
    contradicting that $D = E$.
    Thus a stationary point requires $A_N = \mathcal O(1)$ for a stationary point.

    We now do the same for $A_N = o(1)$.
    First, clearly $n \ge 1$; suppose that in fact $n = \Theta(1)$,
    i.e.\ $A_N = \Theta\left( 1 / (N \sqrt{\log N}) \right)$.
    In this case, we would have
    $D = \frac23 N \log \frac n \delta + \Theta(1)$
    and $E = N + \Theta(1)$,
    so that $D = E$ requires $\frac23 \log \frac n \delta \to 1$,
    i.e.\ $n \to \delta \exp \frac32 \approx 4.5 \, \delta$.
    For $\delta < 0.22$,
    this contradicts $n \ge 1$.
    So we know that $\log n = \omega(1)$.
    Now, the remaining options for $A_N$ all yield $D - E \to \infty$:
    \begin{alignat*}{3}
        & \text{If $A_N = o(1)$, $A_N = \Omega\left( \frac{1}{\log N} \right) $,}
            \quad &&
        D = \Theta\left( N \log n \right)
            , \quad &&
        E = \Theta\left( A_N N \log n \right)
            .\\
        & \text{If $A_N = o\left( \frac{1}{\log N} \right)$, $A_N = \omega\left( \frac{1}{N \sqrt{\log N}} \right)$,}
            \quad &&
        D = \Theta\left( N \log n \right)
            , \quad &&
        E = \Theta\left( N \right)
            .
    \end{alignat*}

    Thus we have established that $A_N = \Theta(1)$.
    Thus, we obtain that
    \[
        D = \frac12 N \log N + \mathcal{O}\left( N \right)
        \qquad
        E = \frac{\sqrt{3} J^*}{2 C} A_N N \log N + \mathcal{O}\left( N \sqrt{\log N} \right)
    .\]
    Asymptotic equality hence requires
    $A_N \to C / (\sqrt{3} J^*)$.
\end{proof}

\subsection{Uniform convergence results} \label{sec:proof:unif-conv}
These results, on the uniform convergence of $\hat\eta$ and $\hat\sigma^2$,
were used in the proof of \cref{thm:ratio-conv}.

\begin{prop}
\label{thm:mmd-conv}
Under \cref{assump:omega-bounded,assump:k-bounded,assump:k-lipschitz},
we have that with probability at least $1 - \delta$,
\begin{equation*}
    \sup_\omega \left\lvert \hat\eta_\omega - \eta_\omega \right\rvert
    \le \frac{8}{\sqrt n} \left[
        \nu \sqrt{2 \log\frac2\delta + 2 D \log\left( 4 R_\Omega \sqrt n \right)}
      + L_k
      \right]
.\end{equation*}
\end{prop}
\begin{proof}
Theorem 7 of \citet{sriperumbudur2009choice} gives a similar bound in terms of Rademacher chaos complexity, but for ease of combination with our bound on convergence of the variance estimator, we use a simple $\epsilon$-net argument instead.

We study the random error function
\[
    \Delta(\omega) := \hat\eta_\omega - \eta_\omega
.\]

First, we place $T$ points $\{ \omega_i \}_{i=1}^T$
such that for any point $\omega \in \Omega$,
$\min_i \norm{\omega - \omega_i} \le q$;
\cref{assump:omega-bounded} ensures this is possible with at most $T = (4 R_\Omega / q)^D$ points
\citep[Proposition 5]{cucker:foundations}.

Now, $\E\Delta = 0$, because $\hat\eta$ is unbiased.
Recall that $\hat\eta = \frac{1}{n (n-1)} \sum_{i \ne j} H_{ij}$,
and via \cref{assump:k-bounded} we know $\abs{H_{ij}} \le 4 \nu$.
This $\hat\eta$, and hence $\Delta$, satisfies bounded differences:
if we replace $(X_1, Y_1)$ with $(X'_1, Y'_1)$,
obtaining $\hat\eta' = \frac{1}{n (n-1)} \sum_{i \ne j} F_{ij}$
where $F$ agrees with $H$ except when $i$ or $j$ is $1$,
then
\begin{align*}
       \abs{\hat\eta - \hat\eta'}
  &\le \frac{1}{n (n-1)} \sum_{i \ne j} \abs{H_{ij} - F_{ij}}
     = \frac{1}{n (n-1)} \sum_{i > 1} \abs{H_{i1} - F_{i1}}
     + \frac{1}{n (n-1)} \sum_{j > 1} \abs{H_{1j} - F_{1j}}
\\&\le \frac{2}{n (n-1)} \sum_{i > 1} 8 \nu
     = \frac{16 \nu}{n}
.\end{align*}
Using McDiarmid's inequality for each $\Delta(\omega_i)$ and a union bound,
we then obtain that with probability at least $1 - \delta$,
\[
    \max_{i \in \{1, \dots, T\}} \abs{\Delta(\omega)}
    \le \frac{16 \nu}{\sqrt{2 n}} \sqrt{\log\frac{2 T}{\delta}}
    \le \frac{8 \nu}{\sqrt{n}} \sqrt{2 \log\frac2\delta + 2 D \log\frac{4 R_\Omega}{q}}
.\]

We also have via \cref{assump:k-lipschitz}, for any two $\omega, \omega' \in \Omega$,
\begin{gather*}
       \abs{\hat\eta_\omega - \hat\eta_{\omega'}}
   \le \frac{1}{n (n-1)} \sum_{i \ne j} \abs{H_{ij}^{(\omega)} - H_{ij}^{(\omega')}}
   \le \frac{1}{n (n-1)} \sum_{i \ne j} 4 L_k \norm{\omega - \omega'}
     = 4 L_k \norm{\omega - \omega'}
\\
       \abs{\eta_\omega - \eta_{\omega'}}
     = \abs{\E\left[H_{12}^{(\omega)}\right] - \E\left[ H_{12}^{(\omega')} \right]}
   \le \E\abs{ H_{12}^{(\omega)} - H_{12}^{(\omega')} }
   \le 4 L_k \norm{\omega - \omega'}
\end{gather*}
so that $\norm{\Delta}_L \le 8 L_k$.
Combining these two results,
we know that with probability at least $1 - \delta$
\[
    \sup_\omega \abs{\Delta(\omega)}
    \le \max_{i \in \{1, \dots, T\}} \abs{\Delta(\omega_i)} + 8 L_k q
    \le \frac{8 \nu}{\sqrt{n}} \sqrt{2 \log\frac2\delta + 2 D \log\frac{4 R_\Omega}{q}} + 8 L_k q
;\]
setting $q = 1 / \sqrt{n}$ yields the desired result.
\end{proof}

\begin{prop} \label{thm:var-conv}
Under \cref{assump:k-bounded,assump:k-lipschitz,assump:omega-bounded},
with probability at least $1 - \delta$,
\begin{equation*}
    \sup_{\omega \in \Omega}
        \left\lvert \hat \sigma_{\omega}^2 - \sigma_{\omega}^2 \right\rvert
    \le \frac{64}{\sqrt n} \left[
        7 \sqrt{2 \log \frac{2}{\delta} + 2 D \log\left(4 R_\Omega \sqrt n \right)}
      + \frac{18 \nu^2}{\sqrt n}
      + 8 L_k \nu
      \right]
.\end{equation*}
\end{prop}
\begin{proof}
We again use an $\epsilon$-net argument
on the (random) error function
\[
    \Delta(\omega)
    := \hat\sigma_{k_\omega}^2 - \sigma_{k_\omega}^2
.\]
First, choose $T$ points $\{ \omega_i \}_{i=1}^T$
such that for any point $\omega \in \Omega$,
$\min_i \norm{\omega - \omega_i} \le q$;
again,
via \cref{assump:omega-bounded} and Proposition 5 of \citet{cucker:foundations}
we have $T \le (4 R_\Omega / q)^D$.
By \cref{thm:var-est-bias,thm:var-est-mcd} and a union bound,
with probability at least $1 - \delta$,
\begin{align*}
       \max_{i \in \{1, \dots, T\}} \abs{\Delta(\omega)}
  &\le 448 \sqrt{\frac2n \log \frac{2 T}{\delta}}
     + \frac{1152 \nu^2}{n}
   \le 448 \sqrt{\frac2n \log \frac{2}{\delta} + \frac2n D \log \frac{4 R_\Omega}{q}}
     + \frac{1152 \nu^2}{n}
.\end{align*}
\Cref{thm:sigma-hat-lip} shows that $\norm{\Delta}_L \le 512 L_k \nu$,
which means that with probability at least $1 - \delta$,
\begin{equation} \label{eq:sup-error-with-q}
    \sup_{\omega \in \Omega} \abs{\Delta(\omega)}
   \le 448 \sqrt{\frac2n \log \frac{2}{\delta} + \frac2n D \log \frac{4 R_\Omega}{q}}
     + \frac{1152 \nu^2}{n}
     + 512 L_k \nu q
.\end{equation}
Taking $q = 1 / \sqrt n$ gives the desired result.
\end{proof}

\begin{lemma} \label{thm:var-est-mcd}
For any kernel $k$ bounded by $\nu$ (\cref{assump:k-bounded}),
with probability at least $1 - \delta$,
\begin{equation*}
    \abs{\hat\sigma_k^2 - \E \hat \sigma_k^2}
    \le 448 \sqrt{\frac{2}{n} \log\frac2\delta}
.\end{equation*}
\end{lemma}
\begin{proof}
    We simply apply McDiarmid's inequality to $\hat\sigma_k^2$.
    Suppose we change $(X_1, Y_1)$ to $(X_1', Y_1')$,
    giving a new $H$ matrix $F$ which agrees with $H$ on all but the first row and column.
    Note that $\abs{H_{ij}} \le 4 \nu$,
    and recall
    \[
        \hat\sigma_k^2 = 4 \left(
            \frac{1}{n^3} \sum_i \left( \sum_j H_{ij} \right)^2
          - \left( \frac{1}{n^2} \sum_{ij} H_{ij} \right)^2
          \right)
    .\]

    The first term in the parentheses of $\hat\sigma_k^2$ changes by
    \begin{align*}
       \abs{\frac{1}{n^3} \sum_i \left( \sum_j H_{ij} \right)^2 - \frac{1}{n^3} \sum_i \left( \sum_j F_{ij} \right)^2}
  &\le \frac{1}{n^3} \sum_{i j \ell} \abs{H_{ij} H_{i\ell} - F_{i j} F_{i \ell}}
    .\end{align*}
    In this sum, if none of $i$, $j$, or $\ell$ are one, the term is zero.
    The $n^2$ terms for which $i = 1$ are each upper-bounded by $32 \nu^2$,
    simply bounding each $H$ or $F$ by $4 \nu$.
    Of the remainder, there are $(n-1)$ terms where $j = \ell = 1$,
    each $\abs{H_{i1}^2 - F_{i1}^2} \le 16 \nu^2$.
    We are left with $2 (n-1)^2$ terms which have exactly one of $j$ or $\ell$ equal to $1$;
    the $j=1$ terms are $\abs{H_{i1} H_{i\ell} - F_{i1} H_{i\ell}} \le \abs{H_{i1} - F_{i1}} \abs{H_{i\ell}} \le (8 \nu) (4 \nu)$,
    so each of these terms is at most $32 \nu^2$.
    The total sum is thus at most
    \[
        \frac{1}{n^3} \left(
            n^2 32 \nu^2
          + (n-1) 16 \nu^2
          + 2 (n-1)^2 32 \nu^2
        \right)
        =
        \left( \frac6n - \frac{7}{n^2} + \frac{3}{n^3} \right) 16 \nu^2
    .\]

    The remainder of the change in $\hat\sigma_k^2$ can be determined by bounding
    \begin{align*}
       \abs{\sum_{ij} H_{ij} - \sum_{ij} F_{ij}}
  &\le \sum_{ij} \abs{H_{ij} - F_{ij}}
     = \sum_j \abs{H_{1j} - F_{1j}}
     + \sum_{i > 1} \abs{H_{i1} - F_{i1}}
\\&\le n (8 \nu) + (n-1) (8 \nu)
     = (8 \nu) (2n - 1)
    ,\end{align*}
    which then gives us
    \begin{align*}
       \abs{\left( \frac{1}{n^2} \sum_{ij} H_{ij} \right)^2 - \left( \frac{1}{n^2} \sum_{ij} F_{ij} \right)^2}
  &  = \abs{\frac{1}{n^2} \sum_{ij} H_{ij} + \frac{1}{n^2} \sum_{ij} F_{ij}}
       \abs{\frac{1}{n^2} \sum_{ij} H_{ij} - \frac{1}{n^2} \sum_{ij} F_{ij}}
\\&\le \left( 2 \cdot 4 \nu \right) \frac{2n-1}{n^2} \left( 8 \nu \right)
     = 64 \nu^2 \left( \frac2n - \frac{1}{n^2} \right)
    .\end{align*}
    Thus
    \begin{align*}
       \abs{\hat\sigma_k^2 - (\hat\sigma'_k)^2}
  &\le 4 \left[ \left( \frac{6}{n} - \frac{7}{n^2} + \frac{3}{n^3} \right) 16 \nu^2 + \left( \frac2n - \frac{1}{n^2} \right) 64 \nu^2 \right]
     = \frac{64 \nu^2}{n^3} \left[ 14 n^2 - 11 n + 3 \right]
   \le \frac{896 \nu^2}{n}
    .\end{align*}
    Because the same holds for changing any of the $(X_i, Y_i)$ pairs,
    the result follows by McDiarmid's inequality.
\end{proof}

\begin{lemma} \label{thm:var-est-bias}
For any kernel $k$ bounded by $\nu$ (\cref{assump:k-bounded}),
the estimator $\hat\sigma_k^2$ satisfies
\begin{equation*}
    \abs{\E \hat\sigma_k^2 - \sigma_k^2}
    \le \frac{1152 \nu^2}{n}
.\end{equation*}
\end{lemma}
\begin{proof}
We have that
\begin{align*}
     \E \hat\sigma_k^2
  &= 4 \left(
         \frac{1}{n^3} \sum_{ij\ell} \E\left[ H_{i\ell} H_{j\ell} \right]
       - \frac{1}{n^4} \sum_{ijab} \E\left[ H_{ij} H_{ab} \right]
     \right)
.\end{align*}

Most terms in these sums have their indices distinct;
these are the ones that we care about.
(We could evaluate the expectations of the other terms exactly, but it would be tedious.)
We can thus break down the first term as
\begin{align*}
     \frac{1}{n^3} \sum_{i j \ell} \E[H_{i \ell} H_{j \ell}]
  &= \frac{1}{n^3} \sum_{i j \ell : \abs{\{ i, j , \ell \}} = 3} \E[H_{i \ell} H_{j \ell}]
   + \frac{1}{n^3} \sum_{i j \ell : \abs{\{ i, j , \ell \}} < 3} \E[H_{i \ell} H_{j \ell}]
\\&= \frac{n (n-1) (n-2)}{n^3} \E[H_{12} H_{13}]
   + \left( 1 - \frac{n (n-1) (n-2)}{n^3} \right) q
,\end{align*}
where $q$ is the appropriately-weighted mean of the various $\E[H_{i\ell} H_{j \ell}]$ terms for which $i, j, \ell$ are not mutually distinct.
Since $\abs{H_{ij}} \le 4 \nu$,
$\E[H_{i\ell} H_{j\ell}] < 16 \nu^2$
and so $\abs{q} \le 16 \nu^2$ as well.
Noting that
\[
    \frac{n (n-1) (n-2)}{n^3} = 1 - \frac3n + \frac{2}{n^2}
\]
we obtain
\begin{equation} \label{eq:h-one-bias}
    \abs{\frac{1}{n^3} \sum_{ij\ell} \E[H_{i\ell} H_{j\ell}] - \E[H_{12} H_{13}]}
    = \left( \frac3n - \frac{2}{n^2} \right) \abs{-\E[H_{12} H_{13}] + q}
    \le \left( \frac3n - \frac{2}{n^2} \right) 32 \nu^2
.\end{equation}

The second term can be handled similarly:
\begin{align*}
     \frac{1}{n^4} \sum_{ijab} \E[H_{ij} H_{ab}]
  &= \frac{1}{n^4} \sum_{ijab : \abs{\{i, j, a, b\}} = 4} \E[H_{ij} H_{ab}]
   + \frac{1}{n^4} \sum_{ijab : \abs{\{i, j, a, b\}} < 4} \E[H_{ij} H_{ab}]
\\&= \frac{n (n-1) (n-2) (n-3)}{n^4} \E[H_{ij} H_{ab}]
   + \left( 1 - \frac{n (n-1) (n-2) (n-3)}{n^4} \right) r
,\end{align*}
where $r$ is the appropriately-weighted mean of the non-distinct terms,
$\abs{r} \le 16 \nu^2$.
For $i, j, a, b$ all distinct,
$\E[H_{ij} H_{ab}] = \E[H_{12}]^2$.
Here
\[
    \frac{n (n-1) (n-2) (n-3)}{n^4}
    = \frac{(n-1) (n^2 - 5 n + 6)}{n^3}
    = 1 - \frac{6}{n} + \frac{11}{n} - \frac{6}{n^3}
\]
and so
\begin{equation} \label{eq:h-sum-bias}
    \abs{\frac{1}{n^4} \sum_{ijab} \E[H_{ij} H_{ab}] - \E[H_{12}]^2}
    \le \left( \frac{6}{n} - \frac{11}{n^2} + \frac{6}{n^3} \right) 32 \nu^2
.\end{equation}

Recalling $\sigma_k^2 = 4 (\E[H_{12} H_{13}] - \E[H_{12}]^2)$,
\[
       \abs{\E \hat\sigma_k^2 - \sigma_k^2}
  \le 128 \nu^2 \left( \frac9n - \frac{13}{n^2} + \frac{6}{n^3} \right)
,\]
and since $n \ge 1$, we have $13 / n^2 > 6 / n^3$, yielding the result.
\end{proof}

\begin{lemma} \label{thm:sigma-hat-lip}
  Under \cref{assump:k-bounded,assump:k-lipschitz},
  we have
  \begin{equation*}
      \sup_{\omega, \omega' \in \Omega}
      \frac{\abs{\hat\sigma_\omega^2 - \hat\sigma_{\omega'}^2}}{\norm{\omega - \omega'}}
      \le 256 L_k \nu
      \qquad\text{ and }\qquad
      \sup_{\omega, \omega' \in \Omega}
      \frac{\abs{\sigma_{\omega}^2 - \sigma_{{\omega'}}^2}}{\norm{\omega - \omega'}}
      \le 256 L_k \nu
  .\end{equation*}
\end{lemma}
\begin{proof}
    We first handle the change in $\hat\sigma_k$:
    \begin{align*}
       \abs{\hat\sigma_{k_\omega}^2 - \hat\sigma_{k_{\omega'}}^2}
  &  = 4 \abs{
        \frac{1}{n^3} \sum_{ij\ell} H_{i\ell}^{(\omega)} H^{(\omega)}_{j \ell}
      - \frac{1}{n^3} \sum_{ij\ell} H_{i\ell}^{(\omega')} H^{(\omega')}_{j \ell}
      - \frac{1}{n^4} \sum_{ijab} H^{(\omega)}_{ij} H^{(\omega)}_{ab}
      + \frac{1}{n^4} \sum_{ijab} H^{(\omega')}_{ij} H^{(\omega')}_{ab}
      }
\\&\le \frac{4}{n^3} \sum_{ij\ell} \abs{
         H_{i\ell}^{(\omega)} H^{(\omega)}_{j \ell}
       - H_{i\ell}^{(\omega')} H^{(\omega')}_{j \ell}
     }
     + \frac{4}{n^4} \sum_{ijab} \abs{
         H^{(\omega)}_{ij} H^{(\omega)}_{ab}
       - H^{(\omega')}_{ij} H^{(\omega')}_{ab}
     }
    .\end{align*}
    We can handle both terms by bounding
    \begin{align*}
       \abs{
         H^{(\omega)}_{ij} H^{(\omega)}_{ab}
       - H^{(\omega')}_{ij} H^{(\omega')}_{ab}
       }
  &\le \abs{
         H^{(\omega)}_{ij} H^{(\omega)}_{ab}
       - H^{(\omega)}_{ij} H^{(\omega')}_{ab}
       } + \abs{
         H^{(\omega)}_{ij} H^{(\omega')}_{ab}
       - H^{(\omega')}_{ij} H^{(\omega')}_{ab}
       }
\\&  = \abs{H^{(\omega)}_{ij}} \abs{H^{(\omega)}_{ab} - H^{(\omega')}_{ab}}
     + \abs{H^{(\omega)}_{ij} - H^{(\omega')}_{ij}} \abs{H^{(\omega')}_{ab}}
\\&\le 4 \nu \left(
       \abs{H^{(\omega)}_{ab} - H^{(\omega')}_{ab}}
     + \abs{H^{(\omega)}_{ij} - H^{(\omega')}_{ij}}
     \right)
    .\end{align*}
    Using \cref{assump:k-lipschitz} and the definition of $H$,
    \[
       \abs{H^{(\omega)}_{ij} - H^{(\omega')}_{ij}}
    \le 4 L_k \norm{\omega - \omega'}
    \]
    so
    \begin{equation} \label{eq:h-step-omega}
       \abs{
         H^{(\omega)}_{ij} H^{(\omega)}_{ab}
       - H^{(\omega')}_{ij} H^{(\omega')}_{ab}
       }
       \le 32 \nu L_k \norm{\omega - \omega'}
    \end{equation}
    and hence
    \begin{align*}
       \abs{\hat\sigma_{\omega}^2 - \hat\sigma_{\omega'}^2}
  &\le 256 \nu L_k \norm{\omega - \omega'}
    .\end{align*}

    Again using \eqref{eq:h-step-omega},
    we also have
    \begin{align*}
       \abs{ \sigma_\omega^2 - \sigma_{\omega'}^2 }
  &\le 4 \abs{\E\left[ H_{12}^{(\omega)} H_{13}^{(\omega)} \right] - \E\left[ H_{12}^{(\omega')} H_{13}^{(\omega')} \right]}
     + 4 \abs{\E\left[ H_{12}^{(\omega)} \right]^2 - \E\left[ H_{12}^{(\omega')} \right]^2 }
\\&\le 4 \E \abs{H_{12}^{(\omega)} H_{13}^{(\omega)} - H_{12}^{(\omega')} H_{13}^{(\omega')}}
     + 4 \E \abs{H_{12}^{(\omega)} H_{34}^{(\omega)} - H_{12}^{(\omega')} H_{34}^{(\omega')}}
\\&\le 256 \nu L_k \norm{\omega - \omega'}
    .\qedhere\end{align*}
\end{proof}

\subsection{Constructing appropriate kernels} \label{sec:proof:kernel-props}

We now show \cref{thm:main-rbf-lip,thm:main-mkl-lip,thm:main-kern-lip},
which each state that 
\cref{assump:k-lipschitz} is satisfied by various choices of kernel.
The following assumption will be useful for different kernel schemes.

\begin{netassumplist}
  \item \label{assump:x-bounded}
    The domain $\X$ is Euclidean and bounded,
    $\X \subseteq \left\{ x \in \R^d : \norm x \le R_X \right\}$ for some constant $R_X < \infty$.
\end{netassumplist}

We begin by recalling a well-known property of the Gaussian kernel, useful for both Gaussian bandwidth selection and deep kernels.
A proof is in \cref{sec:misc-proofs}.
\begin{restatable}{lemma}{gausslip} \label{thm:gauss-lip}
    The Gaussian kernel
    $
        \kappa(a, b) = \exp\left( - \frac{\norm{a - b}^2}{2 \sigma^2} \right)
    $
    satisfies
    \[
        \abs{\kappa(a, b) - \kappa(a', b')}
        \le \frac{1}{\sigma \sqrt{e}} \left( \norm{a - b} + \norm{a' - b'} \right)
        \le \frac{1}{\sigma \sqrt{e}} \left( \norm{a - a'} + \norm{b - b'} \right)
    .\]
\end{restatable}

\subsubsection{Gaussian bandwidth selection (Proposition \ref{thm:main-rbf-lip})} \label{sec:proof:bw-sel}

\Cref{thm:gauss-lip} immediately gives us \cref{assump:k-lipschitz} when we chose among Gaussian kernels:
\begin{prop} \label{thm:rbf-lip}
    Define a one-dimensional Banach space for inverse lengthscales of Gaussian kernels $\gamma > 0$,
    so that $k_\gamma(x, y) = \kappa_{1 / \gamma}(x, y)$,
    with standard addition and multiplication
    and norms defined by the absolute value,
    and $k_0$ taken to be the constant $1$ function.
    Let $\Omega$ be any subset of this space.
    Under \cref{assump:x-bounded},
    \cref{assump:k-lipschitz} holds:
    for any $x, y \in \X$ and $\gamma, \gamma' \in \Gamma$,
    \[
        \abs{k_\gamma(x, y) - k_{\gamma'}(x, y)}
        \le \frac{2 R_X}{\sqrt e} \abs{\gamma - \gamma'}
    .\]
\end{prop}
\begin{proof}
    \[
        \abs{ k_\gamma(x, y) - k_{\gamma'}(x, y) }
        =  \abs{
            \kappa_1\left( \gamma x, \gamma y \right)
          - \kappa_1\left( \gamma' x, \gamma' y \right)
          }
        \le \frac{1}{\sqrt e} \abs{\gamma {\norm{x-y}} - \gamma' {\norm{x-y}}}
          = \frac{\norm{x-y}}{\sqrt e} \abs{\gamma - \gamma'}
    .\qedhere\]
\end{proof}

\subsubsection{Deep kernels (Proposition \ref{thm:main-kern-lip})} \label{sec:proof:deep-kernels}

To handle the deep kernel case, we will need some more assumptions on the form of the kernel.

\begin{netassumplist}[resume]
  \item \label{assump:phi-form}
    $\phi_\omega(x) = \phi_\omega^{(\Lambda)}$
    is a feedforward neural network with $\Lambda$ layers given by
    \[
        \phi_\omega^{(0)}(x) = x
        \qquad
        \phi_\omega^{(\ell)}(x) = \sigma^{(\ell)}\left( W_\omega^{(\ell)} \phi_\omega^{(\ell - 1)}(x) + b_\omega^{(\ell)} \right)
    ,\]
    where the network parameter $\omega$
    consists of all the weight matrices $W_\omega^{(\ell)}$
    and biases $b_\omega^{(\ell)}$,
    and the activation functions $\sigma^{(\ell)}$ are each 1-Lipschitz,
    $\norm{\sigma^{(\ell)}(x) - \sigma^{(\ell)}(y)} \le \norm{x - y}$,
    with $\sigma^{(\ell)}(0) = 0$ so that $\norm{\sigma^{(\ell)}(x)} \le \norm x$.
    Define a Banach space on $\omega$,
    with addition and scalar multiplication componentwise,
    and
    \[
        \norm{\omega} = \max_{\ell \in \{1, \dots, \Lambda\}} \max\left(
                \norm{W_\omega^{(\ell)}}, \norm{b_\omega^{(\ell)}}
            \right)
    ,\]
    where the matrix norm denotes operator norm
    $\norm W = \sup_x \norm{W x} / \norm{x}$.
    (For convolutional networks, see \cref{remark:convnets}.)

  \item \label{assump:k-form}
    $k_\omega$ is a kernel of the form \eqref{eq:deepkernel_simpleForm},
    \[
        k_\omega(x, y) = \left[ (1 - \epsilon) \kappa(\phi_\omega(x), \phi_\omega(y)) + \epsilon \right] q(x, y)
    ,\]
    with
    $0 \le \epsilon \le 1$,
    $\kappa$ a kernel function,
    and $q(x, y)$ a kernel with $\sup_{x} q(x, x) \le Q$.

    Note that this includes kernels of the form $k_\omega(x, y) = \kappa(\phi_\omega(x), \phi_\omega(y))$: take $\epsilon = 0$ and $q(x, y) = 1$.

  \item \label{assump:kappa-lip}
    $\kappa$ in \cref{assump:k-form} is a kernel function satisfying
    \[
        \abs{\kappa(a, b) - \kappa(a', b')}
        \le L_\kappa \left( \norm{a - a'} + \norm{b - b'} \right)
    .\]
    This holds for a Gaussian $\kappa$ via \cref{thm:gauss-lip}.
\end{netassumplist}

We now turn to proving \cref{assump:k-lipschitz} for deep kernels.
First, we will need some smoothness properties of the network $\phi$.
\begin{restatable}{lemma}{dnnlip} \label{thm:dnn-lip}
    Under \cref{assump:phi-form},
    suppose $\omega, \omega'$ have
    $\norm{\omega} \le R$, $\norm{\omega'} \le R$,
    with $R \ne 1$.
    Then, for any x,
    \begin{gather}
        \norm{\phi_\omega(x)}
        \le R^\Lambda \norm x + \frac{R}{R-1} (R^\Lambda - 1)
        \label{eq:net-growth-gen}
        \\
        \norm{\phi_\omega(x) - \phi_{\omega'}(x)}
        \le \left( \Lambda R^{\Lambda-1} \left( \norm x + \frac{R}{R-1} \right) - \frac{R^\Lambda - 1}{(R - 1)^2} \right) \norm{\omega - \omega'}
        \label{eq:net-lip-gen}
    .\end{gather}
    If $R \ge 2$, we furthermore have
    \begin{gather}
        \norm{\phi_\omega(x)}
        \le R^\Lambda (\norm x + 2)
        \label{eq:net-growth}
        \\
        \norm{\phi_\omega(x) - \phi_{\omega'}(x)}
        \le \Lambda R^{\Lambda-1} \left( \norm x + 2 \right) \norm{\omega - \omega'}
        \label{eq:net-lip}
    .\end{gather}
\end{restatable}
The proof, by recursion, is given in \cref{sec:misc-proofs}.
We are now ready to prove \cref{assump:k-lipschitz} for deep kernels.
\begin{prop} \label{thm:kern-lip}
    Make \cref{assump:x-bounded,assump:omega-bounded,assump:k-form,assump:kappa-lip,assump:phi-form},
    with $R_\Omega \ge 2$.\footnote{Of course, if we know a bound of $R_\Omega < 2$, the result will still hold using $R_\Omega = 2$. It is also possible to show a tighter result, via \eqref{eq:net-growth-gen} and \eqref{eq:net-lip-gen} or their analogue for $R = 1$; the expression is simply less compact.}
    Then \cref{assump:k-lipschitz} holds:
    for any $x, y \in \X$
    and $\omega, \omega' \in \Omega$,
    \[
        \abs{k_\omega(x, y) - k_{\omega'}(x, y)}
        \le
        2 Q (1 - \epsilon) L_\kappa \Lambda R_\Omega^{\Lambda-1} (R_X + 2)
        \norm{\omega - \omega'}
    .\]
\end{prop}
\begin{proof}
    \begin{align*}
       \abs{k_\omega(x, y) - k_{\omega'}(x, y)}
  &  = (1 - \epsilon) \abs{
          \kappa(\phi_\omega(x), \phi_\omega(y))
        - \kappa(\phi_{\omega'}(x), \phi_{\omega'}(y))
        } q(x, y)
\\&\le Q (1 - \epsilon) L_\kappa \left(
          \abs{\phi_\omega(x) - \phi_{\omega'}(x)}
        + \abs{\phi_\omega(y) - \phi_{\omega'}(y)}
        \right)
\\&\le Q (1 - \epsilon) L_\kappa
        \Lambda R_\Omega^{\Lambda-1} (\norm x + \norm y + 4) \norm{\omega - \omega'}
\\&\le Q (1 - \epsilon) L_\kappa
        \Lambda R_\Omega^{\Lambda-1} (2 R_X + 4) \norm{\omega - \omega'}
    .\qedhere\end{align*}
\end{proof}

\begin{remark}
For the deep kernels we use in the paper (\cref{assump:k-form,assump:kappa-lip,assump:phi-form}) on bounded domains (\cref{assump:x-bounded}),
we know $L_k$ via \cref{thm:kern-lip};
\cref{thm:test-power-conv} combines \cref{thm:ratio-conv,thm:param-conv,thm:kern-lip}.
If we further use a Gaussian kernel $q$ of bandwidth $\sigma_\phi$,
the last bracketed term in the error bound of \cref{thm:ratio-conv} becomes
\begin{equation*}
    \frac{2 (1-\epsilon)}{\sigma_\phi \sqrt e} \Lambda R_\Omega^{\Lambda-1} (R_X + 2)
    + \sqrt{2 \log \frac{2}{\delta} + 2 D \log\left( 4 R_\Omega \sqrt n \right)}
.\end{equation*}
The component $R_\Omega^{\Lambda-1} (R_X + 2)$, from \eqref{eq:net-growth},
is approximately the largest that $\phi_\omega$ could make its outputs' norms;
$\sigma_\phi$ will generally be on a comparable scale to the norm of the actual outputs of the network,
so their ratio is something like the ``unused capacity'' of the network to blow up its inputs.
This term is weighted about equally in the convergence bound with the square root of the total number of parameters in the network.
\end{remark}

\begin{remark} \label{remark:convnets}
We can handle convolutional networks as follows.
We define $\Omega$ in essentially the same way,
letting $W_\omega^{(\ell)}$ denote the convolutional kernel
(the set of parameters being optimized),
but define $\norm{\omega}$ in terms of the operator norm of the linear transform corresponding to the convolution operator.
This is given in terms of the operator norm
of various discrete Fourier transforms of the kernel matrix
by Lemma 2 of \citet{Bibi2019};
see also Theorem 6 of \citet{sedghi:conv-svs}.
The number of parameters $D$ is then the actual number of parameters optimized in gradient descent,
but the radius $R_\Omega$ is computed differently.%
\end{remark}

\subsubsection{Multiple kernel learning (Proposition \ref{thm:main-mkl-lip})} \label{sec:proof:mkl}

Multiple kernel learning \citep{Gonen:mkl} also falls into our setting.
A special case of this family of kernels was studied for the (easier to analyze) ``streaming'' MMD estimator by \citet{Gretton2012NeurIPS}.%

\begin{netassumplist}[resume]
    \item \label{assump:mkl}
      Let $\{ k_i \}_{i=1}^D$ be a set of base kernels,
      each satisfying $\sup_{x \in \X} k_i(x, x) \le K$ for some finite $K$.
      Define $k_\omega$ as
      \[
        k_\omega(x, y) = \sum_{i=1}^D \omega_i k_i(x, y)
      .\]
      Define the norm of a kernel parameter by the norm of the corresponding vector $\omega \in \R^D$.
      Let $\Omega$ be a set of possible parameters such that for each $\omega \in \Omega$, $k_\omega$ is positive semi-definite, and $\norm\omega \le R_\Omega$ for some $R_\Omega < \infty$.
\end{netassumplist}

Not only does learning in this setting work (\cref{thm:mkl-lip}),
it is also -- unlike the deep setting -- efficient to find an exact maximizer of $\hat J_\lambda$ (\cref{prop:mkl-soln}).

\begin{prop} \label{thm:mkl-lip}
    \Cref{assump:mkl} implies \cref{assump:k-lipschitz,assump:omega-bounded,assump:k-bounded}.
    In particular,
    \begin{gather*}
        \sup_{\omega \in \Omega} \sup_{x \in \X} k_\omega(x, x) \le K R_\Omega \sqrt{D}
        \\
        \abs{k_\omega(x, y) - k_{\omega'}(x, y)}
        \le
        K \sqrt D \norm{\omega - \omega'}
    .\end{gather*}
\end{prop}
\begin{proof}
\Cref{assump:omega-bounded} is immediate from \cref{assump:mkl}, since $\Omega \subset \R^D$.
Let $\mathbf k(x, y) \in \R^D$ denote the vector whose $i$th entry is $k_i(x, y)$,
so that $k_\omega(x, y) = \omega\tp \mathbf k(x, y)$.
As $\norm{\mathbf k(x, y)}_\infty \le K$, we know $\norm{\mathbf k(x, y)} \le K \sqrt D$.
\Cref{assump:k-bounded,assump:k-lipschitz} follow by Cauchy-Schwartz.
\end{proof}

\begin{prop} \label{prop:mkl-soln}
    Take \cref{assump:mkl},
    and additionally assume that $\Omega = \{ \omega \mid \forall i.\, \omega_i \ge 0, \sum_i \omega_i = Q\}$ for some $Q < \infty$.
    A maximizer of $\hat J_\lambda(\omega)$ can then be found
    by scaling the solution to a convex quadratic program,
    \[
        \tilde\omega =
        \argmin_{\omega \in [0, \infty)^D \; : \; \omega\tp \mathbf b = 1}
            \omega\tp (\mathbf A + \lambda I) \omega
        ,\qquad
        \hat\omega = \frac{Q}{\sum_i \tilde\omega_i} \tilde\omega \in \argmax_{\omega \in \Omega} \hat J_\lambda(\omega)
    ,\]
    where
    \begin{gather*}
        \left( \mathbf H_{ij} \right)_\ell
      = k_\ell(X_i, X_j) + k_\ell(Y_i, Y_j) - k_\ell(X_i, Y_j) - k_\ell(X_j, Y_i)
    \\
        \mathbf b
      = \frac{1}{n (n-1)} \sum_{i \ne j} \mathbf H_{ij} 
      \in \R^D
    \\
        \mathbf A
      = \frac{4}{n^3} \sum_i \left( \sum_j \mathbf H_{ij} \right) \left( \sum_j \mathbf H_{ij} \right)\tp
      - \frac{4}{n^4} \left(\sum_{ij} \mathbf H_{ij} \right) \left(\sum_{ij} \mathbf H_{ij} \right)\tp
      \in \R^{D \times D}
    ,\end{gather*}
    as long as $\mathbf b$ has at least one positive entry.
\end{prop}
\begin{proof}
The $H$ matrix used by $\hat\eta_\omega$ and $\hat\sigma_\omega$ takes a simple form:
\begin{align*}
     H_{ij}^{(\omega)}
  &= k_\omega(X_i, X_j) + k_\omega(Y_i, Y_j) - k_\omega(X_i, Y_j) - k_\omega(X_j, Y_i)
   = \omega\tp \mathbf H_{ij}
.\end{align*}
Thus
\begin{align*}
     \hat\eta_{\omega}
  &= \omega\tp \left(\frac{1}{n (n-1)} \sum_{i \ne j} \mathbf H_{ij} \right)
   = \omega\tp \mathbf b
\\   \hat\sigma_\omega^2
  &= \frac{4}{n^3} \sum_i \left( \omega\tp \sum_j \mathbf H_{ij} \right)^2
   - \frac{4}{n^4} \left( \omega\tp \sum_{ij} \mathbf H_{ij} \right)^2
\\&= \omega\tp \left(
        \frac{4}{n^3} \sum_i \left( \sum_j \mathbf H_{ij} \right) \left( \sum_j \mathbf H_{ij} \right)\tp
      - \frac{4}{n^4} \left(\sum_{ij} \mathbf H_{ij} \right) \left(\sum_{ij} \mathbf H_{ij} \right)\tp
      \right) \omega
   = \omega\tp \mathbf A \omega
.\end{align*}
Note that because $\hat\sigma_\omega^2 \ge 0$ for any $\omega$,
we have $\mathbf A \succeq 0$.
We have now obtained a problem equivalent to the one in Section 4 of \citet{Gretton2012NeurIPS};
the argument proceeds as there.
\end{proof}

\subsection{Miscellaneous Proofs} \label{sec:misc-proofs}

The following lemma was used for \cref{thm:rbf-lip,thm:kern-lip}.

\gausslip*
\begin{proof}
    We have that
    \begin{align*}
       \abs{\kappa(a, b) - \kappa(a', b')}
  &  = \abs{\exp\left( - \frac{\norm{a - b}^2}{2 \sigma^2} \right)
          - \exp\left( - \frac{\norm{a' - b'}^2}{2 \sigma^2} \right)}
\\&\le \norm{x \mapsto \exp\left( - \frac{x^2}{2 \sigma^2} \right) }_L
       \abs{\norm{a - b} - \norm{a' - b'}}
    .\end{align*}
    We can bound the Lipschitz constant as its maximal derivative norm,
    \[
        \sup_x \frac{\abs{x}}{\sigma^2} \exp\left( - \frac{x^2}{2 \sigma^2} \right)
    .\]
    Noting that
    \[
        \frac{\mathrm d}{\mathrm d x} \log\left( \frac{\abs{x}}{\sigma^2} \exp\left( - \frac{x^2}{2 \sigma^2} \right) \right)
        = \frac1x - \frac{x}{\sigma^2}
    \]
    vanishes only at $x = \pm \sigma$,
    the supremum is achieved by
    using that value, giving
    \[
        \norm{x \mapsto \exp\left( - \frac{x^2}{2 \sigma^2} \right) }_L
        = \frac{1}{\sigma \sqrt e}
    .\]
    The result follows from
    \[
        \abs{\norm{a - b} - \norm{a' - b'}}
        \le \norm{a - b - a' + b'}
        \le \norm{a - a'} + \norm{b - b'}
    .\qedhere\]
\end{proof}

This next lemma was used in \cref{thm:kern-lip}.

\dnnlip*
\begin{proof}
    First, $\norm{\phi_\omega^{(0)}(x)} = \norm{x}$,
    showing \eqref{eq:net-growth-gen} when $\Lambda = 0$.
    In general,
    \begin{align*}
       \norm{\phi_\omega^{(\ell)}(x)}
  &  = \norm{\sigma^{(\ell)}\left( W_\omega^{(\ell)} \phi_\omega^{(\ell-1)}(x) + b_\omega^{(\ell)} \right)}
\\&\le \norm{W_\omega^{(\ell)} \phi_\omega^{(\ell-1)}(x) + b_\omega^{(\ell)}}
\\&\le \norm{W_\omega^{(\ell)}} \norm{\phi_\omega^{(\ell-1)}(x)} + \norm{b_\omega^{(\ell)}}
\\&\le R \norm{\phi_\omega^{(\ell-1)}(x)} + R
    ,\end{align*}
    and expanding this recursion gives
    \[
        \norm{\phi_\omega^{(\ell)}(x)}
        \le R^\ell \norm{x} + \sum_{m=1}^\ell R^m
        = R^\ell \norm{x} + \frac{R}{R - 1} (R^\ell - 1)
    .\]

    Now, we have \eqref{eq:net-lip-gen} for $\Lambda = 0$
    because $\phi_\omega^{(0)}(x) - \phi_{\omega'}^{(0)}(x) = 0$.
    For $\ell \ge 1$, we have
    \begin{align*}
       \norm{\phi_\omega^{(\ell)}(x) - \phi_{\omega'}^{(\ell)}(x)}
  &  = \norm{
         \sigma^{(\ell)}\left( W_\omega^{(\ell)} \phi_\omega^{(\ell-1)}(x)
       + b_\omega^{(\ell)} \right)
       - \sigma^{(\ell)}\left( W_{\omega'}^{(\ell)} \phi_{\omega'}^{(\ell-1)}(x)
       - b_{\omega'}^{(\ell)} \right)
       }
\\&\le \norm{
         W_\omega^{(\ell)} \phi_\omega^{(\ell-1)}(x)
       - W_{\omega'}^{(\ell)} \phi_\omega^{(\ell-1)}(x)
       } + \norm{
         W_{\omega'}^{(\ell)} \phi_\omega^{(\ell-1)}(x)
       - W_{\omega'}^{(\ell)} \phi_{\omega'}^{(\ell-1)}(x)
       } + \norm{
         b_\omega^{(\ell)}
       - b_{\omega'}^{(\ell)}
       }
\\&\le \norm{ W_\omega^{(\ell)} - W_{\omega'}^{(\ell)} } \norm{\phi_\omega^{(\ell-1)}(x)}
     + \norm{ W_{\omega'}^{(\ell)} } \norm{\phi_\omega^{(\ell-1)}(x) - \phi_{\omega'}^{(\ell-1)}(x)}
     + \norm{ \omega - \omega' }
\\&\le \norm{\omega - \omega'} \left(
          R^{\ell-1} \norm x
        + \frac{R}{R-1} (R^{\ell-1} - 1)
        + 1
        \right)
     + R \norm{\phi_\omega^{(\ell-1)}(x) - \phi_{\omega'}^{(\ell-1)}(x)}
    .\end{align*}
    Expanding the recursion yields
    \begin{align*}
          \norm{\phi_\omega^{(\ell)}(x) - \phi_{\omega'}^{(\ell)}(x)}
        &\le \sum_{m=0}^{\ell-1} R^m \left(
                R^{\ell-1-m} \norm{x}
              + \frac{R}{R-1} (R^{\ell-m-1} - 1) + 1
            \right)
        \norm{\omega - \omega'} 
    \\  &  = \sum_{m=0}^{\ell-1} \left(
                R^{\ell-1} \norm{x}
              + \frac{R^\ell}{R-1}
              - \frac{R^{m+1}}{R-1}
              + R^m
        \right) \norm{\omega - \omega'} 
    \\  &  = \left(
              \ell R^{\ell - 1} \norm x
            + \frac{\ell R^\ell}{R - 1}
            - \left( \frac{R}{R-1} - 1 \right) \sum_{m=0}^{\ell -1} R^m
        \right) \norm{\omega - \omega'}
    \\  &  = \left(
              \ell R^{\ell - 1} \left( \norm x + \frac{R}{R-1} \right)
            - \frac{1}{R-1} \frac{R^\ell - 1}{R - 1}
        \right) \norm{\omega - \omega'}
    .\end{align*}

    When $R \ge 2$, we have that $R / (R - 1) \le 2$ and $R^\ell > 1$,
    giving \eqref{eq:net-growth} and \eqref{eq:net-lip}.
\end{proof}

\section{Experimental Details}\label{Asec:exp_set}

\subsection{Details of synthetic datasets}\label{Asec:syn_intro}

Table~\ref{tab:synthetic datasets} shows details of four synthetic datasets. \emph{Blob} datasets are often used to validate two-sample test methods \citep{Gretton2012NeurIPS,Jitkrittum2016,sutherland:mmd-opt}, although we rotate each blob to show the benefits of non-homogeneous kernels. \emph{HDGM} datasets are first proposed in this paper. \emph{HDGM-D} can be regarded as \emph{high-dimension Blob-D} which contains two modes with the same variance and different covariance.

\begin{table}[ht]
  \centering
  \caption{Specifications of $\P$ and $\Q$ of synthetic datasets. $\mu^b_1 = [0,0], \mu^b_2 = [0,1], \mu^b_3 = [0,2],\dots,\mu^b_8 = [2,1],\mu^b_9 = [2,2]$ (same with Figure~\ref{fig:moti}a). $\mu^h_1 = \textbf{0}_d$, $\mu^h_2 = 0.5 \times \textbf{1}_d$, $I_d$ is an identity matrix with size $d$. $\Delta^b_i=-0.02-0.002\times (i-1)$ if $i<5$ and $\Delta^b_i=0.02+0.002\times (i-6)$ if $i>5$. if $i=5$, $\Delta^b_i=0$ (same with Figure~\ref{fig:moti}a). $\Delta^h_1$ and $\Delta^h_2$ are set to $0.5$ and $-0.5$, respectively.}\label{tab:synthetic datasets}

    \begin{tabular}{lll}
    \toprule
    \rule{0em}{1em}Datasets & $\P$  & $\Q$ \rule{0em}{1em} \\
    \midrule
    \rule{0em}{1.5em}\emph{Blob-S} & $\sum_{i=1}^9\frac19\mathcal{N}({\mu}^b_i,0.03\times I_2)$ & $\sum_{i=1}^9\frac19\mathcal{N}({\mu}^b_i,0.03\times I_2)$ \\
    \rule{0em}{2em}\emph{Blob-D} & $\sum_{i=1}^9\frac19\mathcal{N}({\mu}^b_i,0.03\times I_2)$ & $\sum_{i=1}^9\frac19\mathcal{N}\left({\mu}^b_i,
\begin{bmatrix}
    0.03 & \Delta^b_i \\
    \Delta^b_i & 0.03  \\
\end{bmatrix}\right)$ \\
    \rule{0em}{1.5em}\emph{HDGM-S} & $\sum_{i=1}^2\frac12\mathcal{N}({\mu}^h_i,I_d)$ & $\sum_{i=1}^2\frac12\mathcal{N}({\mu}^h_i,I_d)$ \\
    \rule{0em}{2.2em}\emph{HDGM-D} & $\sum_{i=1}^2\frac12\mathcal{N}({\mu}^h_i,I_d)$ & $\sum_{i=1}^2\frac12\mathcal{N}\left({\mu}^h_i,
\begin{bmatrix}
    1 & \Delta^h_i& \textbf{0}_{d-2}  \\
    \Delta^h_i & 1& \textbf{0}_{d-2}  \\
    \textbf{0}_{d-2}^T & \textbf{0}_{d-2}^T & I_{d-2}
\end{bmatrix}\right)$ \\
    \bottomrule
    \end{tabular}%
  \vspace{-0.5em}
\end{table}%

\subsection{Dataset visualization}\label{Asec:data_visual}

\Cref{fig:MNIST} shows images from real-\emph{MNIST} and ``fake''-\emph{MNIST},
while \cref{fig:CIFAR10} shows samples from \emph{CIFAR-10} and \emph{CIFAR-10.1}.

\begin{figure}[ht]
    \begin{center}
        \subfigure[Real-\emph{MNIST}]
        {\includegraphics[width=0.4\textwidth]{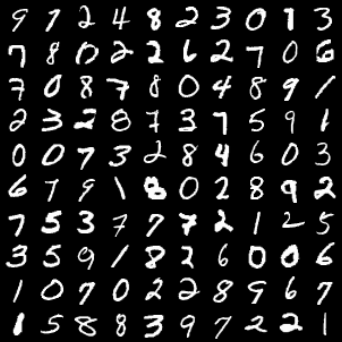}}
        \subfigure[``Fake''-\emph{MNIST}]
        {\includegraphics[width=0.4\textwidth]{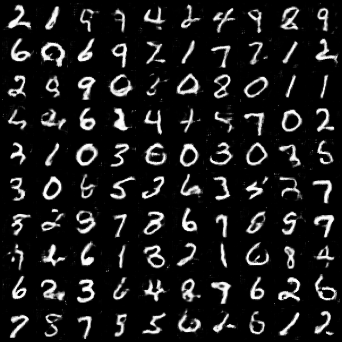}}
        \caption{Images from real-\emph{MNIST} and ``fake''-\emph{MNIST}. ``Fake''-\emph{MNIST} is generated by DCGAN \citep{DCGAN_Radford}.}    \label{fig:MNIST}
    \end{center}
    \vspace{-0.5cm}
\end{figure}

\begin{figure}[!t]
    \begin{center}
        \subfigure[\emph{CIFAR-10} test set]
        {\includegraphics[width=0.4\textwidth]{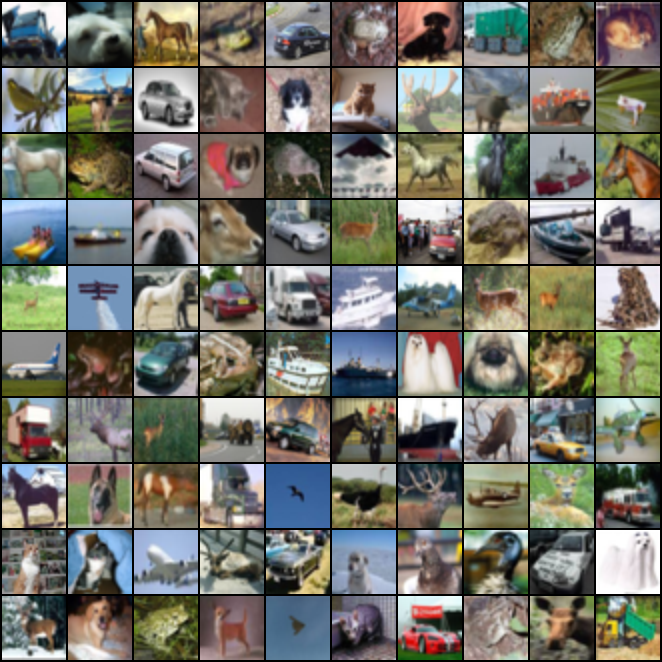}}
        \subfigure[\emph{CIFAR-10.1} test set]
        {\includegraphics[width=0.4\textwidth]{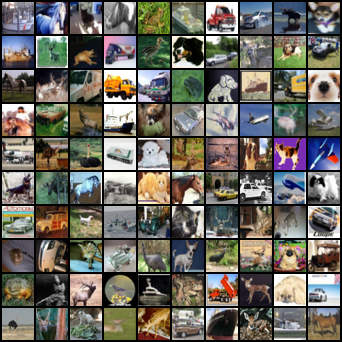}}
        \caption{Images from \emph{CIFAR-10} test set and the new \emph{CIFAR-10.1} test set \citep{recht:imagenet}.}  \label{fig:CIFAR10}
    \end{center}
    \vspace{-0.5cm}
\end{figure}

\subsection{Configurations}\label{Asec:configuration}

We implement all methods on Python 3.7 (Pytorch 1.1) with a NIVIDIA Titan V GPU. We run ME and SCF using the official code \citep{Jitkrittum2016}, and implement C2ST-S, C2ST-L, MMD-D and MMD-O by ourselves. We use permutation test to compute $p$-values of C2ST-S and C2ST-L, MMD-D, MMD-O and tests in Table~\ref{tab:CE_for_TST}. We set $\alpha=0.05$ for all experiments.  Following \citet{Lopez:C2ST}, we use a deep neural network $F$ as the classifier in C2ST-S and C2ST-L, and train the $F$ by minimizing cross entropy.
To fairly compare MMD-D with C2ST-S and C2ST-L, the network $\phi_\omega$ in MMD-D has the same architecture with feature extractor in $F$. Namely, $F = g\circ \phi_\omega$, where $g$ is a two-layer fully-connected network. The network $g$ is a simple binary classifier that takes extracted features (through $\phi_\omega$) as input. For test methods shown in Table~\ref{tab:CE_for_TST}, the network $\phi_\omega$ in them also has the same architecture with that in MMD-D.

For \emph{Blob}, \emph{HDGM} and \emph{Higgs}, $\phi_\omega$ is a five-layer fully-connected neural network. The number of neurons in hidden and output layers of $\phi_\omega$ are set to $50$ for \emph{Blob}, $3\times d$ for \emph{HDGM} and $20$ for \emph{Higgs}, where $d$ is the dimension of samples. These neurons are with softplus activation function, i.e., $\log(1+\exp(x))$. For \emph{MNIST} and \emph{CIFAR}, $\phi_\omega$ is a \emph{convolutional neural network} (CNN) that contains four convolutional layers and one fully-connected layer. The structure of the CNN follows the structure of the feature extractor in the discriminator of DCGAN \citep{DCGAN_Radford} (see \cref{fig:MMDDK_phi,fig:MMD_CIFAR_F} for the structure of $\phi_\omega$ in MMD-D, and \cref{fig:C2ST_F,fig:C2ST_CIFAR_F} for the structure of classifier $F$ in C2ST-S and C2ST-L). The link of DCGAN code is \url{https://github.com/eriklindernoren/PyTorch-GAN/blob/master/implementations/dcgan/dcgan.py}.

We use Adam optimizer \citep{Adam:optimizer} to optimize 1) parameters of $F$ in C2ST-S and C2ST-L, 2) parameters of $\phi_\omega$ in MMD-D and 3) kernel lengthscale in MMD-O. We set drop-out rate to zero when training C2ST-S, C2ST-L and MMD-D on all datasets.

\begin{figure}[!t]
    \begin{center}
        \includegraphics[width=0.95\textwidth]{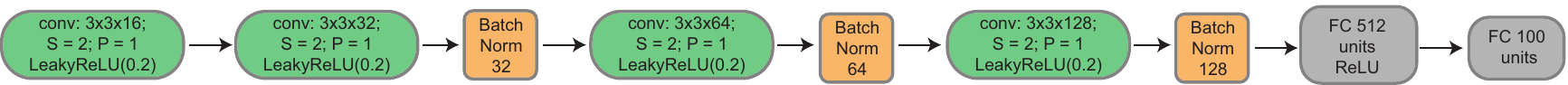}
        \caption{The structure of $\phi_\omega$ in MMD-D on \emph{MNIST}. The kernel size of each convolutional layer is $3$; stride (S) is set to $2$; padding (P) is set to 1. We do not use dropout. Best viewed zoomed in.}\label{fig:MMDDK_phi}
    \end{center}
\end{figure}

\begin{figure}[!t]
    \begin{center}
        \includegraphics[width=0.95\textwidth]{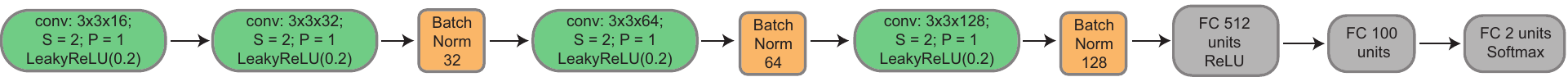}
        \caption{The structure of classifier $F$ in C2ST-S and C2ST-L on \emph{MNIST}. The kernel size of each convolutional layer is $3$; stride (S) is set to $2$; padding (P) is set to 1. We do not use dropout. In the first layer, we will convert the \emph{CIFAR} images from $32\times 32\times 3$ to $64\times 64\times 3$. Best viewed zoomed in.}    \label{fig:C2ST_F}
    \end{center}
\end{figure}

\begin{figure}[!t]
    \begin{center}
        \includegraphics[width=0.95\textwidth]{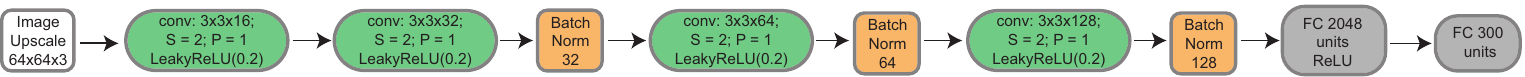}
        \caption{The structure of $\phi_\omega$ in MMD-D on \emph{CIFAR}. The kernel size of each convolutional layer is $3$; stride (S) is set to $2$; padding (P) is set to 1. We do not use dropout in all layers. In the first layer, we will convert the \emph{CIFAR} images from $32\times 32\times 3$ to $64\times 64\times 3$. Best viewed zoomed in.}  \label{fig:MMD_CIFAR_F}
    \end{center}
\end{figure}

\begin{figure}[!t]
    \begin{center}
        \includegraphics[width=0.95\textwidth]{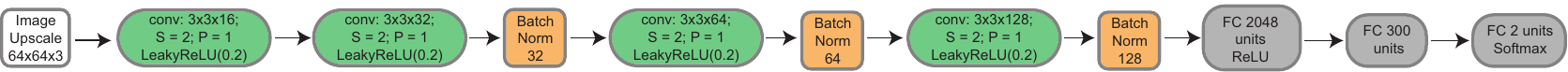}
        \caption{The structure of classifier $F$ in C2ST-S and C2ST-L on \emph{CIFAR}. The kernel size of each convolutional layer is $3$; stride (S) is set to $2$; padding (P) is set to 1. We do not use dropout. Best viewed zoomed in.} \label{fig:C2ST_CIFAR_F}
    \end{center}
\end{figure}

\subsection{Detailed parameters of all test methods}\label{Asec:para_set}

In this subsection, we demonstrate detailed parameters of all test methods. Except for learning rate of Adam optimizer, we use default parameters of Adam optimizer provided by Pytorch. We use one validation set (with the same size of training set) to roughly search these parameters. Using these parameters, we compute test power of each test method on $100$ test sets (with the same size of training set).

For ME and SCF, we follow \citet{Chwialkowski2015} and set $J=10$ for \emph{Higgs}. For other datasets, we set $J=5$.

For C2ST-S and C2ST-L, we set batchsize to $\min\{2\times n_b,128\}$ for \emph{Blob}, $128$ for \emph{HDGM} and \emph{Higgs}, and $100$ for \emph{MNIST} and \emph{CIFAR}. We set the number of epochs to $500\times 18\times n_b/\textnormal{batchsize}$ for \emph{Blob}, $1,000$ for \emph{HDGM}, \emph{Higgs} and \emph{CIFAR}, and $2,000$ for \emph{MNIST}. We set learning rate to $0.001$ for \emph{Blob}, \emph{HDGM} and \emph{Higgs}, and $0.0002$ for \emph{MNIST} and \emph{CIFAR} (following \citet{DCGAN_Radford}).

For MMD-O, we use full batch (i.e., all samples) to train MMD-O. we set the number of epochs to $1,000$ for \emph{Blob}, \emph{HDGM}, \emph{Higgs} and \emph{CIFAR}, and $2,000$ for \emph{MNIST}. We set learning rate to $0.0005$ for \emph{Blob}, \emph{MNIST} and \emph{CIFAR}, and $0.001$ for \emph{HDGM}.

For MMD-D, we use full batch (i.e., all samples) to train MMD-D with samples from \emph{Blob}, \emph{HDGM} and \emph{Higgs}. We use mini-batch (batchsize is $100$) to train MMD-D with samples from \emph{MNIST} and \emph{CIFAR}. We set the number of epochs to $1,000$ for \emph{Blob}, \emph{HDGM}, \emph{Higgs} and \emph{CIFAR}, and $2,000$ for \emph{MNIST}. We set learning rate to $0.0005$ for \emph{Blob} and \emph{Higgs}, $10^{-5}$ for \emph{HDGM}, $0.001$ for \emph{MNIST} and $0.0002$ for and \emph{CIFAR} (following \citet{DCGAN_Radford}). 

\subsection{Links to datasets}
\emph{Higgs} dataset can be downloaded from UCI Machine Learning Repository. The link is \url{https://archive.ics.uci.edu/ml/datasets/HIGGS}.

\emph{MNIST} dataset can be downloaded via Pytorch. See the code in \url{https://github.com/eriklindernoren/PyTorch-GAN/blob/master/implementations/dcgan/dcgan.py}.

\emph{CIFAR-10.1} is available from \url{https://github.com/modestyachts/CIFAR-10.1/tree/master/datasets} (we use \texttt{cifar10.1\_v4\_data.npy}). This new test set contains $2,031$ images from TinyImages \citep{torralba2008tinyimages}.

\subsection{Type I errors on \emph{Higgs} and \emph{MNIST}} \label{sec:typeI}

\begin{table}[t]
\centering
  \small
  \caption{Results on \emph{Higgs} ($\alpha=0.05$). We report average Type I error on \emph{Higgs} dataset when increasing number of samples ($N$). Note that, in \emph{Higgs}, we have two types of Type I errors: 1) Type I error when two samples drawn from $\P$ (no Higgs bosons) and 2) Type I error when two samples drawn from $\Q$ (having Higgs bosons). Type I reported here is the average value of 1) and 2). Since Type I error reported here is the average value of two average Type I errors, we do not report standard errors of the average Type I error in this table.} \label{tab:Higgs_RES2}
\begin{tabular}{c|cccccc}
\toprule
$N$ & ME & SCF & C2ST-S & C2ST-L & MMD-O & MMD-D \\
\midrule
1000 & 0.048 & 0.040 & 0.043 & 0.048 & 0.059 & 0.037\\
2000 &  0.043 & 0.032 & 0.060 & 0.056 & 0.055 & 0.053 \\
3000 & 0.049 & 0.043 & 0.046 & 0.053 & 0.051 & 0.069 \\
5000 & 0.056 & 0.035 & 0.052 & 0.065 & 0.049 & 0.062 \\
8000 & 0.050 & 0.034 & 0.065 & 0.067 & 0.056 & 0.037  \\
10000 & 0.059 & 0.032 & 0.057 & 0.058 & 0.045 & 0.048 \\
\midrule
Avg. &  0.051 & 0.036 & 0.054 & 0.058 & 0.050 & 0.051  \\
\bottomrule
\end{tabular}

\end{table}

Table~\ref{tab:Higgs_RES2} shows average Type I error on \emph{Higgs} dataset when increasing number of samples ($N$). Table~\ref{tab:MNIST_RES2} shows  average Type I error on real-\emph{MNIST} vs. real-\emph{MNIST} when increasing number of samples ($N$).

\begin{table}[!t]
\centering
  \small
  \caption{Results on \emph{MNIST} given $\alpha=0.05$. We report  average Type I error$\pm$standard errors on real-\emph{MNIST} vs. real-\emph{MNIST} when increasing number of samples ($N$).} \label{tab:MNIST_RES2}
\begin{tabular}{c|cccccc}
\toprule
$N$ & ME & SCF & C2ST-S & C2ST-L & MMD-O & MMD-D \\
\midrule
200 &\mnstd{0.076}{0.011} & \mnstd{0.075}{0.010} & \mnstd{0.035}{0.006} & \mnstd{0.045}{0.005} & \mnstd{0.068}{0.004} & \mnstd{0.056}{0.003}  \\
400 & \mnstd{0.062}{0.010} & \mnstd{0.056}{0.007} & \mnstd{0.044}{0.006} & \mnstd{0.040}{0.004} & \mnstd{0.053}{0.005} & \mnstd{0.056}{0.005} \\
600 & \mnstd{0.051}{0.003} & \mnstd{0.049}{0.009} & \mnstd{0.039}{0.005} & \mnstd{0.054}{0.007} & \mnstd{0.066}{0.008} & \mnstd{0.056}{0.008} \\
800 & \mnstd{0.054}{0.006} & \mnstd{0.046}{0.006} & \mnstd{0.043}{0.005} & \mnstd{0.042}{0.007} & \mnstd{0.051}{0.005} & \mnstd{0.054}{0.007} \\
1000 & \mnstd{0.047}{0.006} & \mnstd{0.045}{0.010} & \mnstd{0.038}{0.006} & \mnstd{0.046}{0.005} & \mnstd{0.041}{0.007} & \mnstd{0.062}{0.006} \\
\midrule
Avg. &0.058 & 0.054 & 0.040 & 0.045 & 0.056 & 0.057 \\
\bottomrule
\end{tabular}

\end{table}

\section{Interpretability on \emph{CIFAR-10} vs \emph{CIFAR-10.1}} \label{sec:cifar-interp}

In Section~\ref{sec:benchmark-exp}, we have shown that images in \emph{CIFAR-10} and \emph{CIFAR-10.1} are not from the same distribution. Thus, it is interesting to try to understand the major difference between the datasets. %
Mean Embedding tests \citep{Chwialkowski2015} compare the mean embeddings $\mu_\P$ and $\mu_\Q$ at test locations $v_1, \dots, v_L$, rather than through their overall norm.
The test statistic is
\[
    \hat\Lambda = n \bar z_n\tp S^{-1} \bar z_n
    ,\quad z_i = (k(x_i, v_j) - k(y_i, v_j))_{j=1}^L \in \R^L
    ,\quad \bar z_n = \frac1n \sum_{i=1}^n z_i
    ,\quad S_n = \frac{1}{n-1} \sum_{i=1}^n (z_i - \bar z_n) (z_i - \bar z_n)\tp
;\]
the asymptotic null distribution of $\hat\Lambda$ is $\chi_L^2$,
and the estimator is computable in linear time rather than $\widehat\MMD_U$'s quadratic time.

\Citet{Jitkrittum2017} jointly learn the parameters $v_j$ and kernel parameters to optimize test power.
The best such test locations ($L = 1$) for a Gaussian kernel (with learned bandwidth) are shown in \cref{fig:CIFAR10_interp_learnt_ME}.
We could also try optimizing a deep kernel \eqref{eq:deepkernel_simpleForm} and the test locations together;
this procedure, however, failed to find a useful test.
We can find a better test, though, with a two-stage scheme:
first, learn a deep kernel to maximize $\hat J_\lambda$,
then choose $v_i$ to maximize $\hat\Lambda$ with that kernel fixed.
Results are shown in \cref{fig:CIFAR10_interp_learnt}.

Although these approaches give nontrivial test power,
it is hard to interpret either set of images,
as the test locations have moved far outside the set of natural images.
We can instead constrain $v_1 \in S_\P \cup S_\Q$,
simply picking the single point from the dataset which maximizes $\hat\Lambda$
(shown in \cref{fig:CIFAR10_interp_learnt_SL}).
This achieves similar test power,
but lets us see that the difference might lie in images with smaller objects of interest than the mean for \emph{CIFAR-10}.

\begin{figure}[!p]
    \begin{center}
        \subfigure
        {\includegraphics[width=0.17\textwidth]{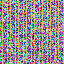}}
        \subfigure
        {\includegraphics[width=0.17\textwidth]{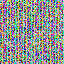}}
        \subfigure
        {\includegraphics[width=0.17\textwidth]{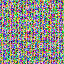}}
        \subfigure
        {\includegraphics[width=0.17\textwidth]{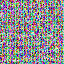}}
        \subfigure
        {\includegraphics[width=0.17\textwidth]{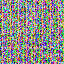}}
        \subfigure
        {\includegraphics[width=0.17\textwidth]{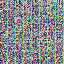}}
        \subfigure
        {\includegraphics[width=0.17\textwidth]{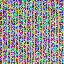}}
        \subfigure
        {\includegraphics[width=0.17\textwidth]{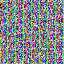}}
        \subfigure
        {\includegraphics[width=0.17\textwidth]{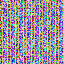}}
        \subfigure
        {\includegraphics[width=0.17\textwidth]{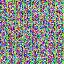}}
        \caption{The best test locations (learned by an ME test with $L=1$) from $10$ experiments on \emph{CIFAR-10} vs \emph{CIFAR-10.1}. Average rejection rate is $0.415$.}  \label{fig:CIFAR10_interp_learnt_ME}
    \end{center}
    \vspace{-0.5cm}
\end{figure}

\begin{figure}[!p]
    \begin{center}
        \subfigure
        {\includegraphics[width=0.17\textwidth]{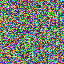}}
        \subfigure
        {\includegraphics[width=0.17\textwidth]{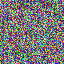}}
        \subfigure
        {\includegraphics[width=0.17\textwidth]{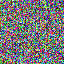}}
        \subfigure
        {\includegraphics[width=0.17\textwidth]{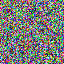}}
        \subfigure
        {\includegraphics[width=0.17\textwidth]{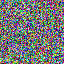}}
        \subfigure
        {\includegraphics[width=0.17\textwidth]{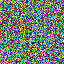}}
        \subfigure
        {\includegraphics[width=0.17\textwidth]{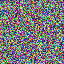}}
        \subfigure
        {\includegraphics[width=0.17\textwidth]{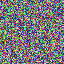}}
        \subfigure
        {\includegraphics[width=0.17\textwidth]{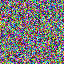}}
        \subfigure
        {\includegraphics[width=0.17\textwidth]{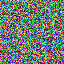}}
        \caption{The best test locations (learned by an ME test, $L=1$, with a deep kernel optimized for an MMD test) from $10$ experiments on \emph{CIFAR-10} vs \emph{CIFAR-10.1}. Average rejection rate is $0.637$.}  \label{fig:CIFAR10_interp_learnt}
    \end{center}
    \vspace{-0.5cm}
\end{figure}

\begin{figure}[!p]
    \begin{center}
        \subfigure
        {\includegraphics[width=0.17\textwidth]{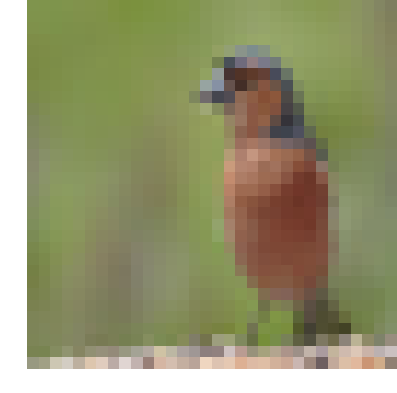}}
        \subfigure
        {\includegraphics[width=0.17\textwidth]{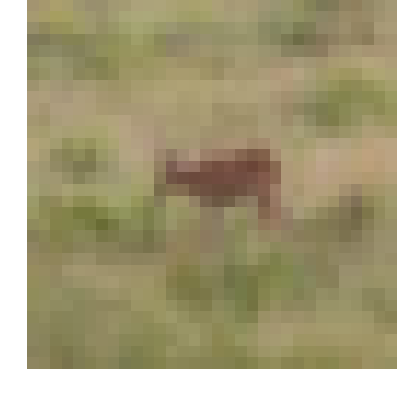}}
        \subfigure
        {\includegraphics[width=0.17\textwidth]{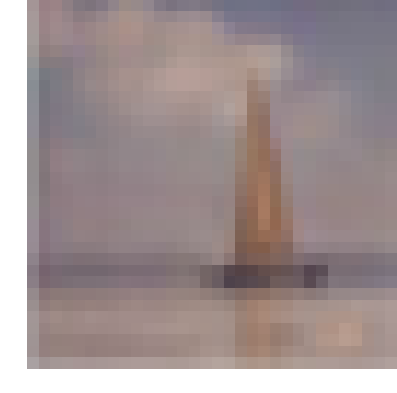}}
        \subfigure
        {\includegraphics[width=0.17\textwidth]{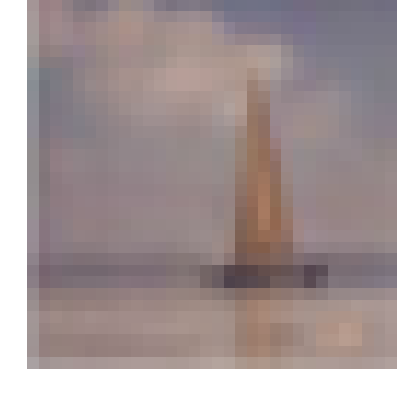}}
        \subfigure
        {\includegraphics[width=0.17\textwidth]{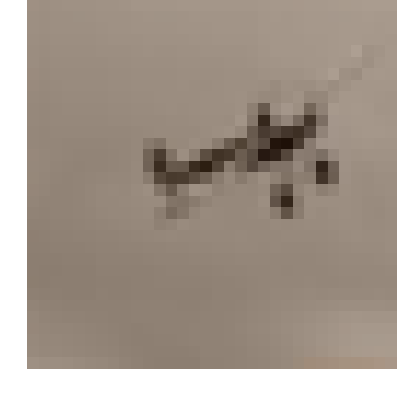}}
        \subfigure
        {\includegraphics[width=0.17\textwidth]{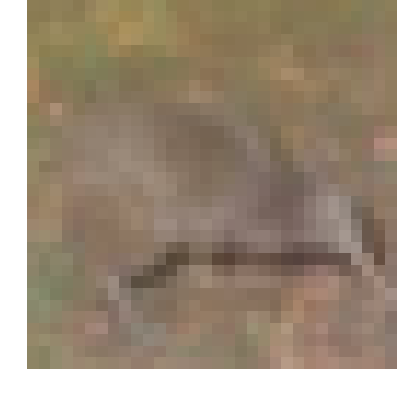}}
        \subfigure
        {\includegraphics[width=0.17\textwidth]{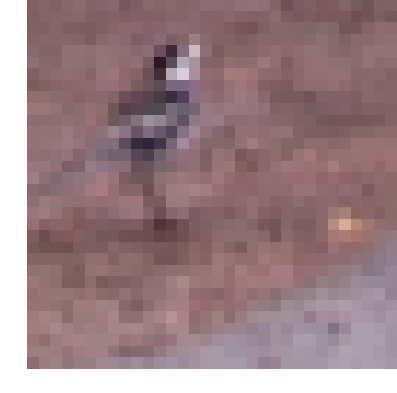}}
        \subfigure
        {\includegraphics[width=0.17\textwidth]{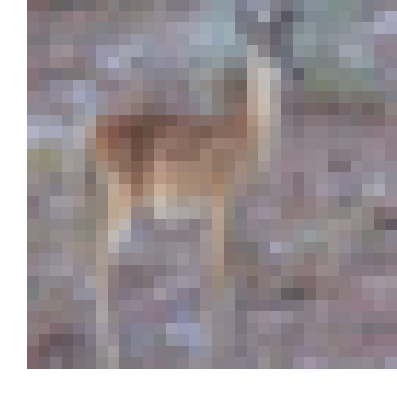}}
        \subfigure
        {\includegraphics[width=0.17\textwidth]{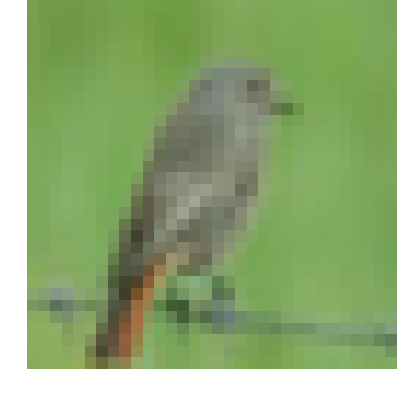}}
        \subfigure
        {\includegraphics[width=0.17\textwidth]{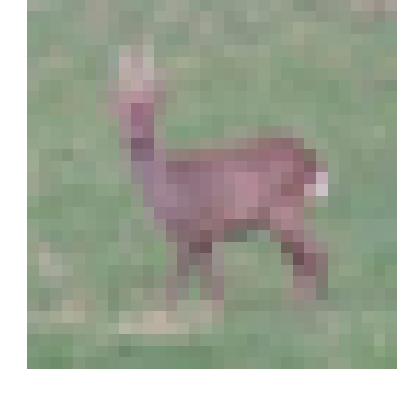}}
        \caption{The best test locations (selected among existing images with our learned deep kernel, $L=1$) from $10$ experiments on \emph{CIFAR-10} vs \emph{CIFAR-10.1}. Average rejection rate is $0.653$.}  \label{fig:CIFAR10_interp_learnt_SL}
    \end{center}
    \vspace{-0.5cm}
\end{figure}

\end{document}